\documentclass[journal]{IEEEtran}
\usepackage{amsfonts}
\usepackage{amsmath, amssymb}
\usepackage{bm}
\usepackage{graphicx}
\usepackage{epstopdf}
\usepackage{subfig}
\usepackage{float}
\usepackage{booktabs}
\usepackage{algorithm}
\usepackage{algorithmic}
\usepackage{mathrsfs}
\usepackage{multirow}
\usepackage{url}
\makeatletter
\g@addto@macro{\UrlBreaks}{\UrlOrds}
\makeatother
\usepackage{amsthm}
\newtheorem{definition}{Definition}
\newtheorem{proposition}{Proposition}

\usepackage{cite}
\ifCLASSINFOpdf
\else
\fi
\begin{document}

\title{Adversarial Diffusion Attacks on Graph-based Traffic Prediction Models}

\author{Lyuyi~Zhu,~\IEEEmembership{}
        Kairui~Feng,~\IEEEmembership{}
        Ziyuan~Pu,~\IEEEmembership{Member, IEEE}
        Wei~Ma,~\IEEEmembership{Member, IEEE}
        % <-this % stops a space
\thanks{L. Zhu is with the College of Civil Engineering and Architecture, Zhejiang University, Hangzhou, China (E-mail: 3170103586@zju.edu.cn).}% <-this % stops a space
\thanks{K. Feng is with the Department of Civil and Environmental Engineering, Princeton University, NJ, U.S.A (E-mail: kairuif@princeton.edu)}
\thanks{Z. Pu is with the School of Engineering, Monash University, Jalan Lagoon Selatan, 47500 Bandar Sunway, Malaysia (Email: ziyuan.pu@monash.edu). }
\thanks{W. Ma is with the Department of Civil and Environmental Engineering, The Hong Kong Polytechnic University, Hong Kong SAR, China (E-mail: wei.w.ma@polyu.edu.hk).}% <-this % stops a space
%\thanks{Manuscript received  \today; revised April 26, 2021.}}
\thanks{Manuscript received: \today.}}

% \markboth{Journal of \LaTeX\ Class Files,~Vol.~14, No.~8, April~2021}%
% {Shell \MakeLowercase{\textit{et al.}}: Bare Demo of IEEEtran.cls for IEEE Journals}
\markboth{Manuscript Submitted to IEEE Internet of Things Journal, 2021}%
{}

\maketitle

\begin{abstract}
Real-time traffic prediction models play a pivotal role in smart mobility systems and have been widely used in route guidance, emerging mobility services, and advanced traffic management systems. With the availability of massive traffic data, neural network-based deep learning methods, especially the graph convolutional networks (GCN) have demonstrated outstanding performance in mining spatio-temporal information and achieving high prediction accuracy. Recent studies reveal the vulnerability of GCN under adversarial attacks, while there is a lack of studies to understand the vulnerability issues of the GCN-based traffic prediction models. Given this, this paper proposes a new task -- diffusion attack, to study the robustness of GCN-based traffic prediction models. The diffusion attack aims to select and attack a small set of nodes to degrade the performance of the entire prediction model. To conduct the diffusion attack, we propose a novel attack algorithm, which consists of two major components: 1) approximating the gradient of the black-box prediction model with Simultaneous Perturbation Stochastic Approximation (SPSA); 2) adapting the knapsack greedy algorithm to select the attack nodes. The proposed algorithm is examined with three GCN-based traffic prediction models: \textsc{St-Gcn}, \textsc{T-Gcn}, and \textsc{A3t-Gcn} on two cities. The proposed algorithm demonstrates high efficiency in the adversarial attack tasks under various scenarios, and it can still generate adversarial samples under the drop regularization such as \textsc{DropOut}, \textsc{DropNode}, and \textsc{DropEdge}. The research outcomes could help to improve the robustness of the GCN-based traffic prediction models and better protect the smart mobility systems.
%, which provides us temporal traffic information for decision making or urban planning. Due to the complexity of city road networks and increasing data scale, traditional methods are difficult to deal with such complex data. Fortunately, graph convolution network (GCN) based models are showing the accuracy and advantages of traffic forecasting and analysis, which also tends to be widely used on navigation maps, APPs, or traffic control centers. Due to the increasing use and reliance on GCN, the security and robustness of these models become significant. However, many recent researches show that GCN has its weakness, which could be fooled and misled easily. In this paper, we explore the robustness of GCN-based traffic forecasting models, which mainly include ST-GCN, T-GCN, and A3T-GCN. We propose a new non-targeted and feature attack task- Diffusion attack. Though traditional adversarial attack methods mainly focus on white-box attack, we solve it by a greedy algorithm and black-box attack method which is based on gradient estimation. The attack method will disturb and mislead the forecasting of GCN under a perturbation of traffic information for a small neighborhood. Besides, we also prove that this method is also effective on models that are strengthened by dropout, dropnode, and dropedge. In other words, the hackers can easily mislead the navigation users by locally attack one marginal device and create global traffic jams.
\end{abstract}


\IEEEpeerreviewmaketitle

\section{Introduction}
\IEEEPARstart{P}{eople's} activities and movements in smart cities rely on accurate, robust, and real-time traffic information. With massive data collected in the intelligent transportation systems (ITS), various methods, such as time series models, state-space models, and deep learning, have been developed to carry out the short-term prediction for traffic operation and management \cite{vlahogianni2004short}. Among these methods, deep learning methods, especially the graph convolutional networks (GCN), achieve state-of-the-art accuracy and are widely employed in industry-level smart mobility applications. For example, Deepmind has partnered with Google Maps to improve the accuracy of real-time Estimated Time of Arrival (ETA) prediction using GCN \cite{deepmind}.

The predicted traffic information plays a critical role in our daily traveling, and travelers take for granted that the predicted results are accurate and trustworthy. However, the robustness and vulnerability issues of these deep learning models have not been investigated for traffic prediction models. Recent studies have shown that neural networks are vulnerable to deliberately designed samples, which are known as adversarial samples. In general, the adversarial samples could be generated by adding imperceptible perturbations to the original data sample. Though the adversarial sample is very similar to its original counterpart, it can significantly change the performance of the deep learning models. %Studies have shown that neural networks are vulnerable to adversarial samples. 
%, thus lead to misclassification or a wrong output. 
Szegedy et al. (2013) firstly discovered this phenomenon on deep neural networks (DNN), and they found that adversarial samples are low-probability but densely distributed \cite{szegedy2013intriguing}. Goodfellow et al. \cite{goodfellow2014explaining} also showed neural networks are vulnerable to the adversarial samples in the sense that it is sufficient to generate adversarial samples when DNNs demonstrate linear behaviors in high-dimensional spaces. 

Due to the existence of the adversarial samples, potential attackers could take advantage of the deep learning models and degrade the model performance. Though related theories and applications have been studied in various areas such as computer vision \cite{su2019one}, social networks \cite{10.5555/3306127.3331707}, traffic signs \cite{9165820} and recommendation systems \cite{Fang_2018}, few of the studies have investigated the vulnerability and robustness of the traffic prediction systems. It has been shown that industry-level traffic information systems can be ``attacked'' easily.  Recently, a German artist walked slowly with a handcart, which was loaded with 99 smartphones. On each smartphone, the mobile application Google Maps was turned on. The 99 cell phones virtually created 99 vehicles on the roads, and all the ``vehicles'' were slowly moving along the road. As Google Maps estimated and predicted the traffic states based on the data sent back from those cell phones, it wrongly identified an empty street (green) to be a congested road (red), as shown in Fig.~\ref{fig:attack}.
Though it is unclear which model Google Maps is using, this experiment indeed reveals the possibility of adversarial attacks on real-world and industry-level traffic information systems.

%turned on Google APP on 99 smart phones, then put them in a handcart and walked slowly to generated virtual traffic jam on Google map (Fig. 1) . Though we don't know what model Google use, it indeed revealed probability of adversarial attack in real world and potential danger of a neural network based model.
% \begin{figure}[!t]
% \centering
% \subfloat[99 Smart Phones in handcart]{\includegraphics[width=0.22\textwidth, height=0.2\textwidth]{2.jpg}}         
% \subfloat[Virtual Traffic Jam On Google Map]{\includegraphics[width=0.22\textwidth, height=0.2\textwidth]{1.jpg}}
% \caption{An example of attack}
% \label{fig:1}
% \end{figure}

\begin{figure*}[ht]
\centering
\includegraphics[width=0.95\textwidth]{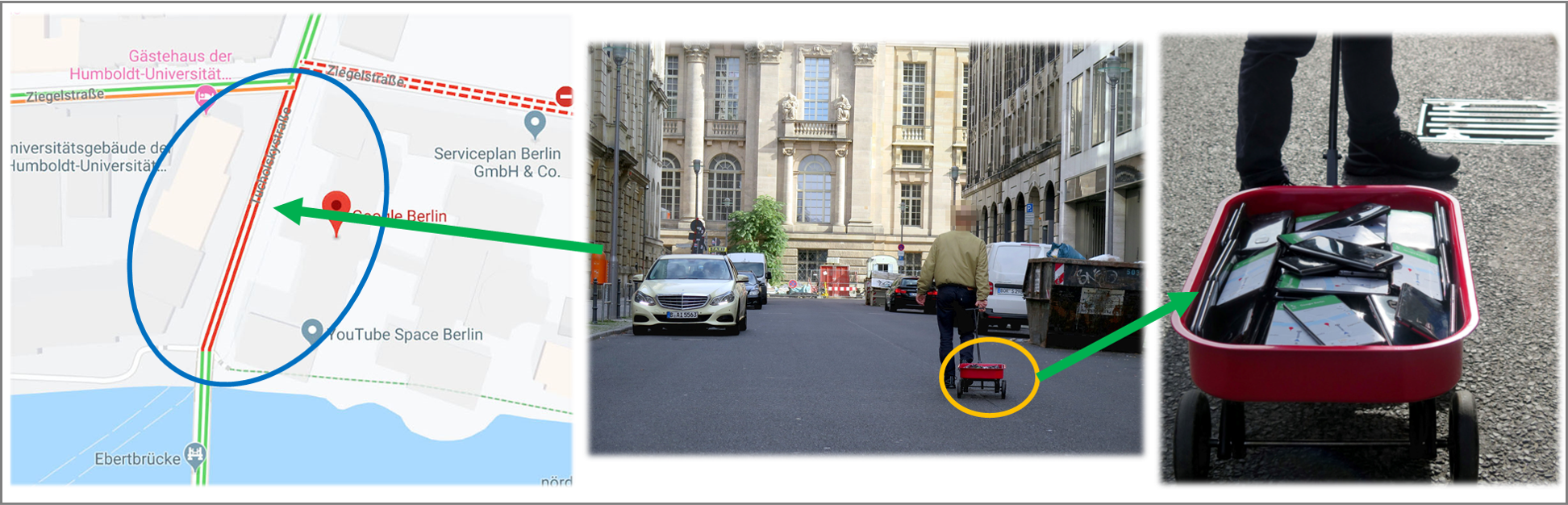}
\caption{Google Maps Hacks using 99 cell phones \cite{simon}.}
\label{fig:attack}
\end{figure*}

%Due to weakness of neural network, potential attacker could take advantages of it, that is adversarial attack. 
%For example, a hacker may hack a machine which is used for traffic data collecting, or change the data during information transfer from a machine to a center or APP. And such attack may cause the paralysis and breakdown of a transportation system. So, improving the robustness of a transportation system and ensuring its security are extremely essential. 

Adversarial attacks on traffic prediction models can affect every aspect of the smart mobility systems. We summarize the following four scenarios in which adversarial attacks can significantly degrade the performance of the systems. 
\begin{itemize}
    \item {\bf Smartphone-based mobility applications.} Mobile phone-based mapping services such as Google Maps and AutoNavi make traffic state estimation and prediction based on the GPS trajectories sent from their users \cite{bayen2011mobile}. However, users' mobile phones can be hacked and the information can be deliberately altered to attack the systems \cite{ahvanooey2020survey}. In general, these individual mobility applications are vulnerable to adversarial attacks because the difficulties of hacking users' mobile phones are way lower than hacking the application server.
    \item {\bf Connected vehicle (CV) systems.}  In the CV systems, traffic information is collected by the roadside units (RSUs) and sent to the traffic control center \cite{6823640}. CVs use information from either RSU or control center to plan routes and avoid congestion. However, it is possible to hack some of the RSUs to send the adversarial samples to manipulate the predicted traffic information from the control center. Attackers could make use of adversarial attacks to benefit a group of vehicles while causing unexpected delays to other vehicles.
    \item {\bf Emerging mobility services.} Transportation network companies (TNCs) depend on accurate traffic speed prediction for vehicle dispatching, routing, and relocation on the central platform. It is possible that a group of vehicles collide with each other and falsely report their GPS trajectories and speed to confuse the central platform. By carefully designing the adversarial samples with purpose, this group of vehicles could take advantage of the central platform by receiving more orders and running on less congested roads. 
    \item {\bf Advanced Traffic Management Systems (ATMS).} Most of the network-wide transportation management systems \cite{8226756} rely on user equilibrium (or stochastic user equilibrium) models to depict and predict travelers' behaviors and both models assume that travelers can acquire accurate (or nearly accurate) traffic information. Under adversarial attacks, this assumption no longer holds. The inaccurately predicted traffic information can reduce the effectiveness of the ATMS and degrade the efficiency of the entire network.
\end{itemize}

%degraded predicted traffic formation can reduce the efficiency of the 
%Hence the accuracy of the  for those .

%group of vehicles collude to benefit 
%special intrest group 

%3.  user optimal are based on the assumption that accurate information is available to travelers network-wide efficiency

Adversarial attack for traffic prediction systems is a unique task that is different from existing literature. In this paper, we focus on the GCN-based traffic prediction models. Previous literature aims at attacking one node by modifying the features of all the nodes in the GCN-based neural networks, while this is impractical on traffic prediction models. On real-world traffic networks, modifying the features on all nodes is challenging and costly and our objective is to degrade the network-wide system performance instead of a single node. A practical task is to degrade the overall system performance by perturbing the features on a small subset of nodes, which can be viewed as the opposite of the conventional adversarial attack problems.
%focuses on a single data sample such as an image, a sentence, or a user. 
To this end, we propose a novel concept -- diffusion attack, and its definition is presented in Definition~\ref{def:diff}.%which aims to attack a small set of nodes in order to degrade the overall performance of the entire system. 
\begin{definition}[Diffusion Attack]
\label{def:diff}
Considering a GCN-based deep neural network, the input features are associated with each node of the corresponding graph. The diffusion attack aims to modify the input features on a limited set of nodes while keeping the graph topology unchanged, and the goal is to degrade the overall performance of the neural network on all the nodes. % This type of attacks is referred to as the diffusion attack.
\end{definition}

An illustration of the diffusion attack is presented in Fig.~\ref{fig:diffuse}. On the left-hand side, the prediction model performs normally and can generate accurate predictions. After the diffusion attack, two nodes in red are attacked, and their nearby neighbors are strongly perturbed, followed by their 2-hop neighbors being slightly perturbed. One can see that the attack effects diffuse from the node being attacked to its neighbors, and later we will develop mathematical proofs and numerical experiments to verify this phenomenon. The purpose of attackers is to select the optimal attack nodes and to generate adversarial samples to maximize attack effects.  

%The reason the attack can be diffused is that: when attacking a single node, the effects diffuse from itself to its neighbors. As shown in Fig.~\ref{fig:diffuse}.

\begin{figure}[ht]
\centering
\includegraphics[height=2.7cm,width=8.8cm]{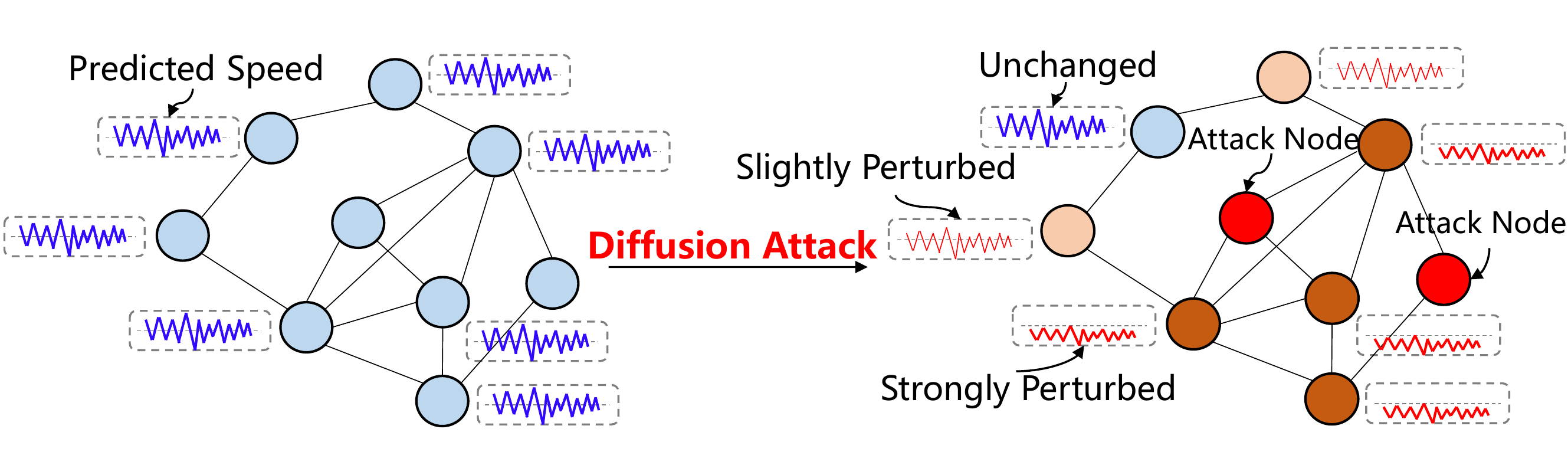}
\caption{An illustration of the diffusion attack.}
\label{fig:diffuse}
\end{figure}

%attacking one node, the effects diffuse along the network
%Unique Diffusion attack for traffic state prediction problem.

%penetrated in every aspects of our daily life 
%F
%New category in the evasion attack
%Availability Attack

%With the development of transport facilities and increasing number of private cars, the intelligent transportation system(ITS) plays a more important role in traffic regulation. The traffic data collected by ITS will be used for forecasting, such as traffic flow, occupancy and speed . The accuracy and robustness of forecasting are significant. 

%Though neural network models substantially improve the accuracy of traffic forecasting and achieve great performance, the robustness of these models are still under debate. 

To summarize, vulnerability issues of the traffic prediction models are critical to smart mobility systems, while the related research is still lacking. Given this, we explore the robustness of the GCN-based traffic prediction models. Based on the unique characteristics of the adversarial attacks on traffic prediction models, we propose the concept of diffusion attack, which aims to attack a small set of nodes to degrade the performance of the entire network. On top of that, we develop a diffusion attack algorithm, which consists of two major components: 1) approximating the gradient of the black-box prediction model with Simultaneous Perturbation Stochastic Approximation (SPSA); 2) adapting the knapsack greedy algorithm to select nodes to attack. The proposed algorithm is examined with three GCN-based traffic prediction models: \textsc{St-Gcn}, \textsc{T-Gcn}, and \textsc{A3t-Gcn} on two cities: Los Angeles and Hong Kong. The proposed algorithm demonstrates high efficiency in the diffusion attack tasks under various scenarios, and the algorithm can still generate adversarial samples under different drop regularization. We further discuss how to improve the robustness of the GCN-based traffic prediction models, and the research outcomes could help to better protect the smart mobility systems in both the cyber and physical world. The contributions of this study are summarized as follows:

%We use a novel black-box attack method, which is based on the greedy algorithm, gradient estimate, and random perturbation. Comparing with other types of black-box attack, it is easy to achieve and have a great performance on GCN.

\begin{itemize}
\item Different from existing adversarial attack tasks on graphs, we propose a novel task of diffusion attack, which aims to select and attack a small set of nodes to degrade the performance of the entire traffic prediction models. This task is suitable for the traffic prediction context, while it is overlooked in the existing literature. 
%{\bf Attack task.} Firstly, we propose a novel black-box attack task- Diffusion Attack. It is an untarget attack and aims to mislead the whole forecasting system. Secondly, though most of research focus on node classification models, we mainly explore the traffic forecasting models, which deal with spatial and temporal information. Thirdly, many research focus on perturbation by changing graph structure and it is not suitable on traffic models. So we focus on feature attack. 
\item To generate the adversarial samples on traffic prediction models, we propose the Simultaneous Perturbation Stochastic Approximation (SPSA) algorithm to efficiently approximate gradients of the black-box prediction models.  
\item To select the optimal attack nodes, we formulate the diffusion problem as a knapsack problem and then adapt the greedy algorithm to determine the priority of the attack nodes.
\item We conduct extensive experiments to attack the widely-used traffic prediction models on Los Angeles and Hong Kong. Different drop regularization strategies ({\em e.g.}, \textsc{DropOut}, \textsc{DropEdge}, \textsc{DropNode}) for defending the adversarial attacks are also tested to ensure that the proposed algorithm can still generate effective adversarial samples in various scenarios. 
%\item {\bf Attack algorithm.}. We compare with the attack task and the 0-1 knapsack problem, and solve it by the combination of greedy algorithm and SPSA method. The proposed attack method shows its great effect on GCN-based traffic forecasting models, which requires less time when deal with high-dimension optimization. 
%Besides, it is also effective on strengthened models by Dropout, Dropnode,and Dropedge.

\end{itemize}

The remainder of this paper is organized as follows. Section \ref{sec:literature} reviews the related studies on traffic prediction, GCN models, and adversarial attacks on graphs.  Section~\ref{sec:model} rigorously formulates the diffusion attack problem and presents the developed attack algorithms. In Sections~\ref{sec:exp}, three traffic prediction models and two real-world datasets are used to examine the proposed attack algorithms. Finally, conclusions and future research are summarized in Section \ref{sec:con}.

\section{Related Works}
\label{sec:literature}
In this section, we first overview the traffic prediction tasks and then summarize recent studies on GCN-based traffic prediction models. Robustness issues of the GCN-based prediction models under adversarial attacks are also discussed.

\subsection{Traffic Prediction}
The traffic prediction problem has been extensively studied for decades, and various statistical models have been developed to solve the problem, including History Average (HA) \cite{edes1980improved}, Autoregressive Integrated Moving Average (ARIMA) \cite{ahmed1979analysis, hamed1995short, van1996combining, lee1999application, williams2003modeling}, Support Vector Regression (SVR) \cite{wu2004travel}, clustering \cite{habtemichael2016short}, and Kalman filtering \cite{okutani1984dynamic, van2011localized}. In recent years, the data scale becomes large and the spatio-temporal correlation of the data becomes complicated, and hence traditional statistical methods reveal their limits in face of the massive and complex data. Instead, neural network models demonstrate potentials in traffic prediction with multi-source data on large-scale networks. Various neural network models have been used for traffic prediction, including Convolutional Neural Network (CNN) \cite{ma2017learning}, Recurrent Neural Network (RNN) \cite{hochreiter1997long, azzouni2017long, ramakrishnan2018network},  attention \cite{guo2019attention, 9003261} and Graph Convolutional Network (GCN) \cite{zhao2019t, yu2018spatio, zhu2020a3t, yu2020forecasting}.

Traffic prediction tasks can also be categorized into multiple purposes, such as traffic state prediction, demand prediction, and trajectory prediction \cite{Ye_2020}. Traffic state prediction includes the prediction of traffic flow \cite{guo2019attention}, speed \cite{zhao2019t}, and travel time \cite{wang2018when}. Traffic demand prediction aims to make prediction of the number of users and traffic demand, such as taxi request \cite{bai2019stg2seq}, subway inflow/outflow \cite{liu2020physicalvirtual}, bike-sharing demand \cite{Lin_2018, 9245516}, and origin-destination demand \cite{8758916}. It is also possible to predict the trajectories of travelers and vehicles, and this task is used for dynamic positioning and resource allocation \cite{monti2020dagnet, mohamed2020socialstgcnn}. Overall, most of the traffic prediction tasks can be carried out by neural network models, and hence it is crucial to study their vulnerability issues. %These data are always complex, includes temporal information and spatial information.  

\subsection{GCN and its Applications on Traffic Prediction}
Traffic data is closely associated with the topological structure of the road networks, and hence it is typical graph-based data. The graph-based data is represented in the non-Euclidean space, and conventional machine learning methods ({\em e.g.,} multi-layer perceptron) overlooks the graph-based inter-relationship among data \cite{bronstein2017geometric}. In this paper, we summarize that traffic data consists of the following two types of information:
%Traditional data structure is Euclidean space, like image, video and voice. 
%However, graph structure is non-Euclidean space \cite{bronstein2017geometric}, such as social network. %Traffic data associates closely with the topological structure of city road network, so it is a typical graph structure. Normally, traffic network consists of two parts:
\begin{itemize}
    \item {\bf Spatial Information.} Traffic-related data can be presented on a graph $\mathcal{G}=(\mathcal{V},\mathcal{E})$, where $\mathcal{V}= \{1, 2, \cdots, N\} $ represents a set of nodes with $N = |\mathcal{V}|$, and $\mathcal{E}$ denotes a set of edges. In traffic prediction, one way to construct the graph is to make each node $v_i$ represent a road segment, and each edge represents the connectivity relationship between the road segments \cite{zhao2019t}. We further define the adjacency matrix $\bm{A} \in \{0,1\}^{N \times N}$ on the graph $\mathcal{G}$, where $A_{ij} = 1$ when node $i$ connects node $j$, and $0$ otherwise.
    %The topological structure of traffic network can be described as a graph $\mathcal{G}=(\mathcal{V},\mathcal{E})$, where $\mathcal{V}= \{v_1,v_2,...,v_N\} $ represents a set of $N$ road nodes. $\mathcal{E}$ denotes a set of edges and represents the intersection condition between road nodes. 
    \item {\bf Temporal Information.} Traffic data, such as traffic speed, density, and flow, on each node can be viewed as a time series. %These data are collected through traffic surveillance sensors. 
    The traffic data on the graph is represented by a feature matrix $\bm{X} \in \mathbb{R}^{N \times S}$, where $S$ denotes the number of time intervals in the study period. %On node $v_i$, $\bm{X}_{i}$ represents the time series data on node $i$, and $\bm{X}_{i, t}$ is the traffic information on node $i$ at time $t$.
    %temporal,  We could collect these data from monitoring machine on each road. These information could be described as a feature matrix $\bm{X} \in \mathbb{R}^{S \times N}$, where $S$ denotes total times interval, $\bm{X}_{t,n}$ represents temporal information of a node $n$ at time $t$.
\end{itemize}

Graph Convolutional Network (GCN) demonstrates great capacities in learning graph-based information for various applications. Short-term traffic prediction is one of the most important and practical applications. The GCN can be used to extract the spatial (non-Euclidean) relationship among nodes \cite{zhao2019t}, and it can couple with recurrent neural networks (RNN) to learn the spatio-temporal patterns of the traffic data. For example, Long Short-Term Memory (LSTM) is widely used with GCN for traffic prediction \cite{MA2015187}, and the Gated Recurrent Unit (GRU) is also adopted to model the time-series on each node \cite{zhao2019t}. Some emerging models for time-series modeling, such as gated CNN and attention, can also be incorporated into the GCN-based deep learning framework \cite{yu2018spatio, zhu2020a3t, guo2019attention, yu2020forecasting}. Readers are referred to 
\cite{Ye_2020} for a comprehensive review of the GCN-based traffic prediction methods.

%These collected data will be sent to a traffic command center for traffic decision making or a transportation navigation APP for traffic forecasting, and aims to improve traffic condition in a city.
%Due to unique non-Euclidean structure
%\cite{kipf2016semi}. 
%Recently, There are many GCN-based models which are used to learn traffic data and make prediction, like . 

As the backbone of a GCN-based model, the graph convolutional layer is defined as follows:
\[
\bm{H}^{(l+1)}=\sigma(\hat{\bm{A}} \bm{H}^{(l)} \bm{W}^{(l)}),
\]
where $l$ is the layer index, $\bm{H}^{(0)}=\bm{X}$, and $\hat{\bm{A}} = \widetilde{\bm{D}}^{-\frac{1}{2}} \widetilde{\bm{A}} \widetilde{\bm{D}}^{-\frac{1}{2}}$ is the Laplacian matrix. $\widetilde{\bm{D}}$ is the corresponding degree matrix $\widetilde{\bm{D}}_{ii} = \sum_{j} \bm{A}_{ij}$, and $\widetilde{\bm{A}} = \bm{A} + \bm{I}_N$, where $\bm{I}_N$ is a $N \times N$ identity matrix. $\sigma(\cdot)$ represents the activation function and $\bm{W}^{(l-1)}$ are parameters of the $l$th layer.

A $L$-layers GCN model can be expressed as following:
\begin{equation}
\label{eq:f}
\bm{Y} = f(\bm{X},\bm{A})=g(\hat{\bm{A}}\cdots \sigma(\hat{\bm{A}}\bm{X}\bm{W}^{(0)})\cdots\bm{W}^{(L-1)}),
\end{equation}
where $g(\cdot)$ is a generalized function, $\bm{Y} \in \mathbb{R}^{N\times T}$ and the parameters $\bm{W}^{(l-1)}$ for each layer could be learned by minimizing the loss between the estimated $\bm{Y}$ and true $\bm{Y}_{true}$, represented by $\mathcal{L}\left(\bm{Y}, \bm{Y}_{true}\right)$. This paper will adopt the GCN as the backbone model and study its robustness and vulnerability issues under adversarial attacks.

%Mean Squared Error (MSE) between the estimated $\bm{Y}$ and true $\bm{Y}_{true}$, represented by $\mathcal{L} = \left\|\bm{Y} - \bm{Y}_{true}\right\|_2^2$. 

%\begin{equation}
%min \ \mathcal{L} = \left\|\bm{Y} - \bm{Y}_{true}\right\| + \lambda L_{reg}
%\end{equation}
%Where $\bm{Y}_{true}$ is true speed of each road in traffic network, $L_{reg}$ is a regularization term and $\lambda$ is a hyperparameter.
%advanced approaches

\subsection{Adversarial Attacks On Graphs}
Like other neural networks, GCN is vulnerable to adversarial attacks. Various attack concepts on graph have proposed, such as targeted and non-targeted attacks, structure and feature attacks, poisoning and evasion attacks. Targeted attacks aim to attack a target node \cite{dai2020targeted}, while the non-targeted attacks aim to compromise global performance of a model \cite{zugner2019adversarial, dai2018adversarial, ma2020practical, ma2021nearblackbox}. Structure attacks modify the structure of graph, such as adding or deleting nodes/edges \cite{dai2018adversarial, xu2019topology, 10.1145/3427228.3427245}, and feature attacks perturb the labels/features of nodes without changing the connectivity of graph \cite{zugner2019adversarial, liu2021one, finkelshtein2020singlenode}. Poisoning attacks modify the training data \cite{zugner2019adversarial}, and evasion attacks insert an adversarial samples when using the models \cite{dai2018adversarial, wang2020evasion, liu2021one, ma2020practical}. To be specific, attacks on traffic prediction systems are non-targeted, feature, and evasion attacks, which have not been well studied in the existing literature. 

The attack algorithms can be further categorized into  white-box attacks and black-box attacks. In white-box attacks, the neural network structure and parameters, training methods, and training samples are exposed to attackers. Attackers could utilize the neural network model to generate adversarial samples. Many white-box methods achieve great performance, such as fast gradient sign method (FGSM) \cite{goodfellow2014explaining}, Jacobian-based saliency map approach (JSMA) \cite{papernot2016limitations}, Carlini and Wagner Attacks (CW) \cite{carlini2017towards},  and Deepfool \cite{moosavi2016deepfool}. In contrast, in black-box attacks, attackers know little about internal information of the target neural network model, especially the model structure and parameters, and only the model input and output are exposed to the attackers. For example, if attackers plan to attack the traffic prediction system, he/she is unlikely to know the internal structure or information of the prediction model. 

Though it is more common in the real world, the black-box attack is much more challenging than the white-box attack. Black-box attacks can be achieved through response surface models \cite{papernot2017practical} and meta-heuristic. For instance, one-pixel attack \cite{su2019one} uses Differential Evolution (DE) algorithm and decision-based attack \cite{brendel2017decision} uses Covariance Matrix Adaptation Evolutionary Strategies (CMA-ES) to generate and improve the adversarial samples by iteration. It is also possible to approximate the gradient of the target models, and the representative models include Zeroth Order Optimization (ZOO) \cite{chen2017zoo}, Autoencoder-based Zeroth Order Optimization Method (AutoZOOM) \cite{tu2019autozoom}, and Natural Evolutionary Strategies (NES) \cite{ilyas2018black}. Besides, the semi-black-box attacks are in between, and it is assumed that the information of the prediction models is partially observed \cite{akhtar2018threat}. The above attack algorithms mainly focus on the classification task, while attack methods for traffic prediction models ({\em i.e.} regression task) are still lacking.

%To summarize, it is more realistic to develop black-box or semi-black-box attack algorithms for GCN-based traffic prediction models, while the corresponding research is still lacking.

%The first type of black-box attack is substitute model training , which means getting a white-box model by training. The structure of this white-box model is similar with the target black-box model. Then attacker could conduct white-box attack on this substitute model and obtain adversarial samples. These samples could transfer to target black-box model and work. The second type is evolution algorithm. The key idea of these methods is generating better solution or population by iteration. 

%In addition, Due to importance of the robustness of GCN, there are many research analyze the robustness of graph \cite{zugner2019certifiable, 10.1145/3394486.3403217, bojchevski2019certifiable, wang2020certified}. And meanwhile many defense strategies have proposed \cite{Wang2019GraphDefenseTR, zhang2020gnnguard, 10.1145/3292500.3330851}.

%\subsection{Robustness Proof And Defense Strategy On Graph}

\section{Proposed Works}
\label{sec:model}
In this section, we first present the general formulation of the adversarial attack on graphs. Then, the new concept of diffusion attack is proposed for traffic prediction models. Lastly, we propose the new formulation and algorithm to construct the diffusion attack and discuss its implementations. 

\subsection{Preliminaries}
Traffic prediction on graphs can be regarded as a graph-based regression problem \cite{Ye_2020}. Using $f: \mathbb{R}^{N \times S} \rightarrow \mathbb{R}^{N \times T}$ in Equation~\ref{eq:f} as a regression model ({\em e.g.,} a traffic prediction model) on graph $\mathcal{G}$, and $f$ contains a $L$-layer GCN, where $S$ represents the look-back time interval, and $T$ is the number of time intervals to predict. The feature matrix $\bm{X}$ contains the historical traffic states, and $\bm{Y}$ is the traffic states we want to predict. %The output $\bm{Y}$ can be expressed as
%\[
%  \bm{Y}(\bm{X}) = f(\bm{X}, \bm{A})=\sigma\left(\bm{A}\cdots %\sigma(\bm{A}\bm{X}\bm{W}^{(1)})\cdots\bm{W}^{(l)}\right),
%\]
%where $\bm{A}$ is the adjacency matrix, and $\{\bm{W}^{(l')}\}_{l'=1,\cdots,l}$ are trainable parameters. 
For traffic prediction problems, $\bm{Y}$ mainly depends on $\bm{X}$ as most of $\bm{A}$ are fixed \cite{Ye_2020}. We further denote $\bm{X} = (\bm{x}_1; \cdots; \bm{x}_N)$, $\bm{Y}= f(\bm{X}) = (\bm{y}_1; \cdots; \bm{y}_N)$, and $\bm{x}_i$ and $\bm{y}_i$ are the $i$th row of $\bm{X}$ and $\bm{Y}$, respectively. For node $i$, $\bm{x}_i$ is the $i$th row of the feature matrix $\bm{X}$, and $\bm{x}_i$ represents a time series of speed on node $i$. The prediction model for node $i$ can be written as $\bm{x}_i = (x_{i1}, x_{i2}, \cdots, x_{iS}) \mapsto \bm{y}_i = (y_{i1}, y_{i2}, \cdots, y_{iT})$, where $x_{i\cdot}$ is the historical traffic states, and $y_{i\cdot}$ is the future traffic states on node $i$. 

\subsection{Diffusion Attacks on Traffic State Forecasting Models}
This paper aims to attack the traffic prediction system by adding perturbations on the feature matrix $\bm{X}$, to maximally change the prediction results over a selected set of nodes. As discussed above, the graph structure is fixed and cannot be changed easily for the problem of traffic prediction, so we assume that $\bm{A}$ is fixed throughout the paper.

%An attacker usually changes the data by adding perturbation. 
We construct an adversarial sample by adding the perturbation $\bm{U}$, which is the same size as $\bm{X}$, to the original input feature matrix $\bm{X}$, as represented by $\bm{X}' = \bm{X} + \bm{U}$. Consequently, the corresponding output changes from $\bm{Y} = f(\bm{X})$ to $\bm{Y}' = f(\bm{X} + \bm{U})$.

%, will be added to the original input $\bm{X}$, represented by   which called an adversarial sample. $\bm{u}$ . Consequently, the corresponding output of model changes from $\bm{Y}(\bm{X})$ to $\bm{Y}^{'}(\bm{X} + \bm{u})$. It can be defined as a GCN Attack.

We suppose that attackers select a set of nodes $\mathcal{P} \subseteq \mathcal{V}$ to attack. For each node $i \in \mathcal{P}$, $\bm{u}_i
\begin{cases} 
\neq \bm{0}& i \in \mathcal{P}\\ 
= \bm{0}& i \notin \mathcal{P} 
\end{cases}$, where $\bm{u}_i$ is the $i$th row of $\bm{U}$. Then adversarial sample can be expressed as follows:
\begin{equation}
\bm{X}'= \bm{X} + \bm{U} = (\bm{x}_1 + \bm{u}_1; \cdots; \bm{x}_{i} + \bm{u}_{i}; \cdots; \bm{x}_{N} + \bm{u}_{N}). \nonumber
\end{equation}

%Given a set of nodes to be attacked: $\mathcal{P} = \{p_1, p_2,..., p_k\}$. For node $i$, its perturbation vector 

%Then the total perturbation is $\bm{u}=(\bm{u}_1,..., \bm{u}_i, ..., \bm{u}_N)$. 

%\subsection{Attack Influence Function}
By perturbing the node $i$, we change the original prediction result on node $i$ (denoted as $\bm{y}_{i}$) to $\bm{y}'_{i}$. The attack influence function $\phi_i(\bm{U})$ on node $i$ is defined as follows:
\begin{equation} \label{eq1}
  \phi_i(\bm{U}) = \mathcal{L}\left( \bm{y}'_i(\bm{X}'), \bm{y}_i (\bm{X}) \right),
\end{equation}
where we use $\bm{y}_i (\bm{X})$ to indicate that $\bm{y}_i$ is a function of $\bm{X}$. $\mathcal{L}(\cdot,\cdot)$ represents loss function between $\bm{y}'_i$ and $\bm{y}_i$. 
Here the attack influence evaluates the difference between original prediction (instead of true speed) and perturbed speed. The reason is that we assume the prediction is accurate enough, otherwise there is no need to attack. On the other hand, we could never know the true traffic condition in the future, so it is impossible to perturb the prediction against true values. 

To mathematically characterize the diffusion phenomenon when attacking the GCN, we demonstrate that the following Proposition~\ref{prop:diff} holds.
\begin{proposition}
\label{prop:diff}
Using the $L$-layer GCN model presented in Equation~\ref{eq:f} for traffic prediction, the effect of perturbation $\bm{U}$ on each node $i$, which is denoted as $\phi_i(\bm{U})$, depends on the perturbations of its $L$-hop neighbors. 
\end{proposition}
\begin{proof} 
Using $|\cdot|$ to represent the element-wise absolute value operator, and assuming that $g(\cdot)$ is Lipschitz continuous with constant $M$, $[\cdot]_i$ is the $i$th row of a matrix, $\sigma(\cdot)$ is the \texttt{Relu} function, and $\mathcal{L}$ represents the Mean Squared Error (MSE), we have
{\small
\begin{equation}
\begin{array}{lllllll}
\phi_i(\bm{U})&=& \left\| \bm{y}'_i(\bm{X}') - \bm{y}_i (\bm{X}) \right\|_2^2\\
&=& \left\| \left[g \! \left( \! \hat{\bm{A}} \! \bm{H}'^{(L-1)} \! \bm{W}^{(L-1)} \! \right)  - \! g \! \left(\! \hat{\bm{A}} \! \bm{H}^{(L)} \! \bm{W}^{(L-1)} \! \right)  \right]_i \right\|_2^2\\
&\leq& M \left\|  \left[ \hat{\bm{A}}(\bm{H}'^{(L-1)} -\bm{H}^{(L-1)} )\bm{W}^{(L-1)} \right]_i \right\|_2^2\\
&\leq&M  \left\|  \left[\hat{\bm{A}}\left(|\bm{H}'^{(L-1)} - \bm{H}^{(L-1)}|\right)|\bm{W}^{(L-1)}| \right]_i \right\|_2^2\\
&\leq&M  \left\|  \left[\hat{\bm{A}}\left(|\bm{H}'^{(L-1)} - \bm{H}^{(L-1)}|\right) \right]_i \right\|_2  \left\| |\bm{W}^{(L-1)}| \right\|_2^2\\
&\leq&MW  \left\|  \left[\hat{\bm{A}}\left(|\bm{H}'^{(L-1)} - \bm{H}^{(L-1)}|\right) \right]_i \right\|_2^2  \\
%&\leq&M  \left\|  \left[|\hat{\bm{A}}|^2\left(|\bm{H}'^{(L-1)} - \bm{H}^{(L-1)}|\right) \right]_i \right\|_2 \cdot  \\
%&&\left\| |\bm{W}^{(L)}| \right\|_2\left\| |\bm{W}^{(L-1)}| \right\|_2\\
&\leq&MW^2 \left\|  \left[\hat{\bm{A}}^2\left(|\bm{H}'^{(L-2)} - \bm{H}^{(L-2)}|\right) \right]_i \right\|_2^2 \\
&\leq& MW^{L} \left\|  \left[\hat{\bm{A}}^{L}\left(|\bm{H}'^{(0)} - \bm{H}^{(0)}|\right) \right]_i \right\|_2^2\\
&=&MW^{L} \left\| \hat{\bm{A}}^{L}_{i \cdot}|\bm{U}| \right\|_2^2,
\end{array}\nonumber
\end{equation}
}
where $\left\| |\bm{W}^{(l)}| \right\|_2^2 \leq W$ and $\hat{\bm{A}}_{i \cdot}^{L}$ represents the $i$th row of $\hat{\bm{A}}^{L}$. Here we can see that the upper bound of attack influence $\phi_i(\bm{U})$ associates closely with $\left\| \hat{\bm{A}}^{L}_{i \cdot}|\bm{U}| \right\|_2^2$. Denoting $|\bm{U}|_{\cdot k}$ as the $k$th column of $|\bm{U}|$, we can expand $\left\| \hat{\bm{A}}^{L}_{i \cdot}|\bm{U}| \right\|_2^2$ as follows:
\begin{equation}
\begin{array}{ll}
\left\| \hat{\bm{A}}^{L}_{i \cdot}|\bm{U}| \right\|_2^2 &= \left\| \left[\hat{\bm{A}}^{L}_{i \cdot} |\bm{U}|_{\cdot 1}, \cdots ,\hat{\bm{A}}^{L}_{i \cdot} |\bm{U}|_{\cdot k},\cdots, \hat{\bm{A}}^{L}_{i \cdot} |\bm{U}|_{\cdot S}\right] \right\|_2^2\\
&=\sum_{k=1}^{S} (\hat{\bm{A}}^{L}_{i \cdot} |\bm{U}|_{\cdot k})^2
\nonumber
\end{array}
\end{equation}
%where $\hat{\bm{A}}^{L}_{i \cdot} |\bm{U}|_{\cdot k}$ is $k$th time component of attack upper bound. 
For each $k$ we can further expand it as the sum of perturbation on different nodes:
\begin{equation}
\begin{array}{ll}
\hat{\bm{A}}^{L}_{i \cdot} |\bm{U}|_{\cdot k} = \sum_{h=1}^{N} \hat{\bm{A}}^{L}_{ih} |\bm{U}|_{hk},
\nonumber
\end{array}
\end{equation}
where $|\bm{U}|_{hk}$ is the absolute value of perturbation added on $k$th historical traffic state of node $h$. $\hat{\bm{A}}^{L}_{ih}$ represents normalized connectivity weight between node $i$ and node $h$ (similar to $\bm{A}^L_{ih}$, which represents number of $L$-hop paths between $i$ and $h$, according to the graph theory). If $\hat{\bm{A}}^{L}_{ih}$ is large, node $h$ will have more impact on $\phi_i(\bm{U})$. Specially, if node $h$ is out of $L$-hop neighbors of node $i$, then $\hat{\bm{A}}^{L}_{ih} \equiv 0$, which means that attack effect will not diffuse to node $i$ when attacking $h$. For most of GCN-based traffic prediction models $L \leq 3$, so the effect of $\bm{U}$ only diffuses to its local neighbors.

%TODO One can see that $\left\| \bm{y}'_i(\bm{X}') - \bm{y}_i (\bm{X}) \right\|_2 \leq MW^{L+1} \left\| |\hat{\bm{A}}|^{(L+1)}_i|\bm{U}| \right\|_2$, which indicates that the perturbation $\bm{U}$ only affects the $L$-hop neighbors of node $i$. For most of GCN-based traffic prediction models $L \leq 3$, so the effect of $\bm{U}$ only diffuses to its local neighbors.
\end{proof}

%To demonFor a $L$-layer GCN model $\bm{Y}(\bm{X})=\hat{\bm{A}}...\sigma(\hat{\bm{A}}\bm{X}\bm{W}^{(0)})...\bm{W}^{(L)}$

%We could give a simple derivation of upper bound of SPSA attack on a $L$-layers GCN model $\bm{Y}(\bm{X})=\hat{\bm{A}}...\sigma(\hat{\bm{A}}\bm{X}\bm{W}^{(0)})...\bm{W}^{(L)}$, where $\sigma(\cdot) = Relu = max(0, x)$. Denote $\|\cdot\|$ is the norm operator of matrix, we have:
% \begin{equation}
%     \begin{aligned}
%         &\|\bm{Y}(\bm{X}') - \bm{Y}(\bm{X})\| = \|\hat{\bm{A}}(\bm{H}'^{(L)} - \bm{H}^{(L)})\bm{W}^{(L)}\|\\ 
%         &\le \|\hat{\bm{A}}\| \cdot \|\bm{H}'^{(L)} - \bm{H}^{(L)}\| \cdot \|\bm{W}^{(L)}\| \\ 
%         &\le  M_{L}\|\hat{\bm{A}}\| \cdot \|\hat{\bm{A}}(\bm{H}'^{(L - 1)} - \bm{H}^{(L - 1)})\bm{W}^{(L - 1)}\|\cdot \|\bm{W}^{(L)}\| \\ 
%         &\le M_L\|\hat{\bm{A}}\|^2\cdot\|\bm{H}'^{(L-1)} - \bm{H}^{(L-1)}\|\cdot\|\bm{W}^{(L-1)}\|\cdot \|\bm{W}^{(L)}\| \\
%         & \le \cdots \le  \prod_{i=1}^{L}M_i \|\hat{\bm{A}}\|^{L+1}\cdot\|\bm{X}' - \bm{X}\| \cdot \prod_{j=0}^{L} \|\bm{W}^{(j)}\| \\
%         &= MAW\|\bm{U}\| = K\|\bm{U}\| \le \varepsilon K \|\bm{X}\| \nonumber
%     \end{aligned}
% \end{equation}
%where $M_i \le 1$ is the Lipschitz constant of $i$-th layer, and $M = \prod_{i=1}^{L}M_i \le \ 1$. $W = \prod_{j=0}^{L} \|\bm{W}^{(j)}\|$, $A=\|\hat{\bm{A}}\|^{L+1}$. And in a trained model, $K = M \times A \times W$ is a constant. $\epsilon$ is a constraint showed in Eq~\ref{eq:diff}. 

%As shown in Equation~\ref{eq1}, $\phi_i(\bm{U})$ measures the effects of the perturbations on changing the prediction results. 
As discussed in the previous section, this paper focuses on the novel concept of diffusion attack, which aims at changing the network-wide prediction results by perturbing a small subset of node features. Mathematically, the diffusion attack problem can be expressed as to find the optimal $\mathcal{P}$ and the corresponding perturbation $\bm{U}$ such that the following influence function $\Phi(\bm{U})$ is maximized:
%As Eq.(\ref{eq1}) shows, it is a target attack for node q. 
%However, only attacking a node does not seem ambitious for some attackers. As Fig. (1) shows, a person with a handcraft could mislead traffic forecasting on a road or an area. What if there are a group of these persons? If they walk on different roads with their handcrafts in the same time, can they mislead the whole traffic forecasting system? We could call it diffusion attack, which means mislead the whole network by poisoned few nodes. This problem could be expressed as: given $\mathcal{P}$ to be poisoned (or attacked), and find a perturbation $\bm{u}$ to maximize the following evaluation function $g(\bm{u})$:

\begin{equation} \label{eq:diff}
\begin{array}{llllll}
%\max \ &\underset{\mathcal{P}\{p_1, p_2,..., p_k\}}{g(\bm{u})} = \sum\limits_i^N b_i g_i(\bm{u}) \\
\max\limits_{\bm{U}, \bm{z}} && \Phi(\bm{U}) = \sum_{i \in \mathcal{V}} w_i \phi_i(\bm{U}) &\\
s.t.& &-\varepsilon^- z_i \bm{x}_{i} \leq \bm{u}_{i} \leq \varepsilon^+ z_i \bm{x}_{i} &\forall i \in  \mathcal{V}\\
&& \sum_{i \in \mathcal{V}} b_i z_i \leq B\\
&& z_i \in \{0,1\} &\forall i \in \mathcal{V}
\end{array}
\end{equation}
where $z_i$ indicates whether node $i$ is attacked. To be precise, $z_i = \begin{cases}  1& i \in \mathcal{P}\\ 0& i \notin \mathcal{P} 
\end{cases}$. $b_i$ represents the cost of attacking node $i$, and $B$ is the total budget. The objective function $\Phi(\bm{U})$ represents attack influence function for the entire network, and $w_i$ is a pre-determined importance weight of node $i$. $\varepsilon^-, \varepsilon^+ >0$ are used to control the scale of the perturbations. 

Formulation~\ref{eq:diff} is a mixed integer programming (MIP), and it contains two components: 1) determining $\bm{U}$ given a fixed $\mathcal{P}$; 2) determining $\mathcal{P}$. The two components will be discussed in section~\ref{sec:spsa} and \ref{sec:greedy}, respectively.

%As an optimization problem, the convexity of the objective function is basic and important. 

%, and those busy downtown roads may have a higher value. $\epsilon$ represents a constraint constant. The Eq. (\ref{eq2}) denotes that if some nodes are attacked, the influence will lead to a diffusion on graph, just likes a infectious disease  (Fig. 2) .

% \begin{figure}[!t]
% \centering
% \includegraphics[width=0.42\textwidth, height=0.2\textwidth]{6.png}
% \caption{Diffusion Attack}
% \label{fig_sim}
% \end{figure}
\subsection{Black-box attacks using SPSA}
\label{sec:spsa}

Given a fixed $\mathcal{P}$, Equation~\ref{eq:diff} reduces to a continuous optimization problem with affine constraints, as shown in the following equation:
\begin{equation} \label{eq:spsa}
\begin{array}{llllll}
 \max\limits_{\bm{U}} && \Phi(\bm{U}) = \sum_{i \in \mathcal{V}} w_i \phi_i(\bm{U}) &\\
s.t.& &-\varepsilon^-  \bm{x}_{i} \leq \bm{u}_{i} \leq \varepsilon^+  \bm{x}_{i} &\forall i \in  \mathcal{P}\\
    % & &\bm{u}_{j}=\bm{0} &\forall j \notin \mathcal{P} \\
\end{array}
\end{equation}

Formulation~\ref{eq:spsa} suggests that an ideal perturbation $\bm{U}^{\ast}$ should be small, and meanwhile it maximizes the overall influence function $\Phi(\bm{U}^{\ast})$. In real-world applications, the internal information about the traffic prediction model is opaque, hence it is proper to consider the traffic prediction model as a black-box. We assume both input feature $\bm{X}$ and prediction results $\bm{Y}$ are known to attackers as both matrices represent the true and predicted traffic states in real-world, and hence Formulation~\ref{eq:spsa} can be viewed as a black-box optimization problem. To solve it, we adopt the Simultaneous Perturbation Stochastic Approximation (SPSA) method, which is featured with its efficiency and scalability \cite{spall1998overview}. In recent years, SPSA has been used for adversarial attacks on classification problems \cite{pmlr-v80-uesato18a}, while it has not been used for regression problems, in particular, the traffic prediction problem. 

%owever, an attacker normally know nothing about internal information about traffic system, so it's a black-box model. 

%We could solve the optimization problem by  . SPSA is able to solve high-dimension optimization problems in an efficient way \cite{pmlr-v80-uesato18a}. The main idea of SPSA is to estimate the gradient of a black box.
The SPSA method uses finite differences between two randomly perturbed inputs to approximate the gradient of the objective function. Mathematically, the gradient of $\Phi$ can be calculated as follows:
\begin{equation}
\label{eq:approx}
  \widehat{\nabla \Phi}(\bm{U}_n) = \frac{\Phi(\bm{U}_n + c_n \bm{\Delta}_n) - \Phi(\bm{U}_n - c_n \bm{\Delta}_n)}{2 c_n \bm{\Delta}_n},
\end{equation}
where $n$ is the index of the iteration, $\bm{\Delta}_n$ is a random perturbation vector whose elements are sampled from Rademacher distribution (Bernoulli $\pm1$ distribution with probability $p=0.5$, and we denote $\textsc{Rad}_{1 \times S} \in \{-1, 1\}^{1 \times S}$ as a sample vector that follows Rademacher distribution). We further denote sequences $\{ c_n \}$ and $\{a_n\}$ as follows:
\begin{equation} \label{eq:spsaseq}
a_n = \frac{a}{(\eta + n) ^ \alpha} \quad\quad c_n = \frac{c}{n^ \gamma},
\end{equation}
where $a,c,\alpha,\gamma$ are hyper-parameters for SPSA \cite{spall1998overview}. Both sequences decrease when the iteration $n$ increases. Then the gradient ascent approach is utilized to maximize $\Phi(\bm{U})$, as shown in the following equation:
\begin{equation}
\label{eq:update}
  \bm{U}_{n+1} = \bm{U}_n + a_n\widehat{\nabla \Phi}(\bm{U}_n)\quad \forall n.
\end{equation}

To summarized, the adversarial attack with fixed node set $\mathcal{P}$ is presented in Algorithm~\ref{alg:spsa}.

% We try to find optimal solution $\bm{u}^{\ast}$:
% \begin{equation}
%   \bm{u}^{\ast} = argmax \ g(\bm{u}) 
% \end{equation}
% Then we use the iterative process as following:

% Where $\hat{g}_n(\bm{u}_n)$ represents estimate of gradient of function g at $\bm{u}_n$, and positive series $\{ a_n \}$ converging to $0$.

%Then the $i^{th}$ component of the estimated gradient $\hat{g}_n(\bm{u}_n)$ can be calculated as following:

% Where . 

\begin{algorithm} 
\caption{Determine the adversarial sample $\bm{X}'$ and optimal perturbation $\bm{U}$ given a fixed $\mathcal{P}$}
\begin{algorithmic}
\REQUIRE Traffic prediction model $f(\bm{X})$, input feature matrix $\bm{X}$, attack set $\mathcal{P}$, maximum iteration $\texttt{MaxIter}$.
\ENSURE Adversarial sample $\bm{X}'$, and optimal perturbation $\bm{U}$. 
\STATE Initialize $\bm{U}_1 \in \mathbb{R}^{N \times S}$
\FOR{$n = 1, 2, \cdots, \texttt{MaxIter}$} 
  \STATE Update $a_n$ and $c_n$ based on Equation~\ref{eq:spsaseq}.
  \STATE For $i\in\mathcal{V}$, random sample $\bm{\delta}_i$ = $\begin{cases}
                        \bm{0}_{1 \times S} &  \text{if $i \notin \mathcal{P}$} \\
                        \textsc{Rad}_{1 \times S}  & \text{if $i \in \mathcal{P}$}
                     \end{cases}$
  \STATE $ \bm{\Delta} = (\bm{\delta}_1; \cdots; \bm{\delta}_i; \cdots; \bm{\delta}_N)$.
  \STATE $\bm{U}^+ \leftarrow \bm{U}_n + c_n \bm{\Delta}; \bm{U}^- \leftarrow \bm{U}_n - c_n \bm{\Delta}$.
  \STATE Compute $\widehat{\nabla \Phi}(\bm{U}_n)$ based on Equation~\ref{eq:approx}.
  \STATE Compute $\bm{U}_{n+1}$ based on Equation~\ref{eq:update}.
  \STATE Set $(\bm{u}_1; \cdots; \bm{u}_i; \cdots; \bm{u}_N)  \leftarrow \bm{U}_{n+1}$.
  \FOR{$i \in \mathcal{V}$}
    \IF{$i \in \mathcal{P}$}
    \STATE $\bm{u}_i \leftarrow \min(\epsilon^+ \bm{x}_i , \bm{u}_i)$
    \STATE $\bm{u}_i \leftarrow \max (-\epsilon^- \bm{x}_i , \bm{u}_i)$
    \ELSE
    \STATE $\bm{u}_i \leftarrow  \bm{0}$
    \ENDIF
      %\STATE $Indexmax \ = \ where \ {\bm{u}^{new}_j > \epsilon \bm{x}_{j}}$
      %\STATE $Indexmin \ = \ where \ {\bm{u}^{new}_j < -\epsilon \bm{x}_{j}}$
      %\STATE $\bm{u}^{new}_j(Indexmax) \leftarrow \epsilon \bm{x}_{j}(Indexmax)$
  \ENDFOR
  \STATE $\bm{U}_{n+1} \leftarrow(\bm{u}_1; \cdots; \bm{u}_i; \cdots; \bm{u}_N)$
\ENDFOR
\STATE $\bm{U} \leftarrow \bm{U}_{\texttt{MaxIter}+1}$
\STATE $\bm{X'} \leftarrow \bm{X} + \bm{U}$
\STATE Return $\bm{X'}, \bm{U}$
\end{algorithmic}
\label{alg:spsa}
\end{algorithm}

\subsection{Node Selection using Knapsack Greedy}
\label{sec:greedy}

This section focuses on determining the attack set $\mathcal{P}$ with a limited budget $B$. The cost $b_i$ is different across different nodes due to the level of difficulty in attacking the node. For example, urban roads may contain a more recent and secured information collection system ($b_i$ is high), while rural roads can be attacked easily ($b_i$ is low). In contrast, attacks on urban roads usually generate a higher impact on the traffic prediction methods because the urban traffic volumes are high. It can be seen that there is a trade-off between the cost and benefit when selecting the attack set $\mathcal{P}$. %To quantify the attack influence, we first denote $\Psi(\mathcal{P})$ as the optimal objective in Formulation~\ref{eq:spsa} given a fixed $\mathcal{P}$. Mathematically, $\Psi(\mathcal{P})$ is represented as follows:

We review the formulation of the diffusion attack in Equation~\ref{eq:diff}, and it is actually similar to the 0-1 knapsack problem, except for that the utility of each node $i$ $\left(\phi_i(\bm{U})\right)$ is unknown \cite{martello1999dynamic}.  The attack set $\mathcal{P}$ can be viewed as a knapsack with maximum capacity $B$, and the nodes are items with their weight $w_i$. Node $i$ is added to $\mathcal{P}$ if $z_i = 1$. Due to the nature of integer programming, there is no provably efficient method to solve formulation~\ref{eq:diff}, which is a NP-hard problem. Real-world networks contain hundreds or thousands of nodes, and hence it is impractical to enumerate all possible integer solutions. In this paper, we develop a family of Knapsack Greedy (\textsc{Kg}) algorithms to solve for formulation~\ref{eq:diff}, and those algorithms are inspired by the original greedy algorithm for the knapsack problem. 

A trivial but insightful observation is that any perturbation $\mathbf{U}$ could reduce the performance of the prediction model. Proposition~\ref{prop:convex} shows that the convexity exists if we only perturb one node and the perturbation $u$ is small enough, then the object function in Formulation~\ref{eq:diff} is locally convex.
\begin{proposition}
\label{prop:convex}
The objective function $\Phi$ in Formulation~\ref{eq:diff} is locally convex under small perturbation $\bm{U}$.
\end{proposition}
\begin{proof} Given function $\Phi$ is smooth and attains global optimal when the perturbation $\bm{U} = \bm0$, there exists one region around $\bm{U} = \bm0$, in which $\Phi(\bm{U})$ is convex.
\end{proof}
Existing literature has shown that the convex separable nonlinear knapsack problems could be approximated with the greedy search framework described by Algorithm ~\ref{alg:gcg}  \cite{zhang2008unified}. Recalling Proposition~\ref{prop:diff}, the perturbation on GCN-based models would only arise local effect, which means if the attack nodes are selected L-hop away from each other, the objective of Formulation 3 would be separable. We also observed that the objective is locally convex given attack on only one node in Proposition~\ref{prop:convex}. Combing the enlightenment we get from Proposition ~\ref{prop:diff} and ~\ref{prop:convex}, the greedy search framework would work efficiently for solving formulation~\ref{eq:diff}.

%We propose to use \textsc{Spsa} 
The proposed solution procedure consists of two steps: 1) compute $\hat{\phi}_i$ to approximate $\phi_i$ for each $i$; 2) adopt the greedy algorithm for the standard knapsack problem with $\hat{\phi}_i$ as the utility. %To be precise, we define the utility of adding node $i$ to $\mathcal{P}$ as $\hat{\phi}_i$, which represents the marginal effect of node $i$ per unit cost.
%Higher utility indicates more beneficial of adding $i$ to $\mathcal{P}$. 
In Step 1, we proposed that the utility $\hat{\phi}_i$ can be obtained by \textsc{Spsa}. To be precise, we run Algorithm~\ref{alg:spsa} with $\mathcal{P} = \mathcal{V}$, and the algorithm outcome is $\bm{U}_{\mathcal{V}}$. Then, $\phi_i$ is approximated by $\hat{\phi}_i = \phi_i(\bm{U}_{\mathcal{V}})$; in Step 2, we initialize the attack set $\mathcal{P}$ as an empty set, then each node is added to the attack set sequentially with the highest utility over budget. %As the feasible solution of $\bm{U}$ is enlarged by adding nodes to $\mathcal{P}$ iteration by iteration, 
The entire procedure is referred as \textsc{Kg-Spsa}, and details of the algorithm are presented in Algorithm~\ref{alg:gcg}.

\begin{algorithm}
\caption{\textsc{Kg-Spsa} for the diffusion attack on traffic prediction models.}
\begin{algorithmic}
\REQUIRE Traffic prediction model $f(\bm{X})$, input feature matrix $\bm{X}$, total budget $B$, and cost of each node $\{b_i\}_{i \in \mathcal{V}}$.
\ENSURE Adversarial sample $\bm{X}'$, optimal perturbation $\bm{U}$, and attack set $\mathcal{P}$.
\STATE Initialize $\mathcal{P} = \emptyset$.
\STATE Evaluate $\hat{\phi}_i = \phi_i(\bm{U}_{\mathcal{V}}), i\in \mathcal{V}$ with Algorithm~\ref{alg:spsa}.
\WHILE{$\sum\limits_{i \in \mathcal{P}} b_i \leq B$}
  \STATE Set $\texttt{max\_utility}=-\infty, \texttt{max\_idx}=-\infty$.
  \FOR{$i \in \mathcal{V}\setminus\mathcal{P}$}
  %\STATE Evaluate $\frac{\Psi(i|\mathcal{P})}{b_i}$ based on Equation~\ref{eq:marg}.
    \IF{$\frac{\hat{\phi}_i}{b_i} > \texttt{max\_utility}$}
       \STATE Set $\texttt{max\_utility} = \frac{\hat{\phi}_i}{b_i}$.
       \STATE Set $\texttt{max\_idx} = i$.
    \ENDIF
  \ENDFOR
  \IF {$\sum\limits_{i \in \mathcal{P}} b_i + b_{\texttt{max\_idx}}\leq B$}
    \STATE{$\mathcal{P} = \mathcal{P} \cup \{\texttt{max\_idx}\}$}
  \ENDIF
\ENDWHILE
\STATE{Run Algorithm~\ref{alg:spsa} with fixed $\mathcal{P}$, obtain $\bm{X}'$, $\bm{U}$.}
\STATE Return $\bm{X'}, \bm{U}, \mathcal{P}$.
\end{algorithmic}
\label{alg:gcg}
\end{algorithm}

It is possible to adopt other methods to estimate $\hat{\phi}_i$ in Step 1, including clustering methods and graph centrality measures. These algorithms will be viewed as baseline algorithms and compared with \textsc{Kg-Spsa} in numerical experiments.

%To make a simplification and accelerate the algorithm, we could calculate $\frac{\Psi(i|\mathcal{P})}{b_i}$ by a further approximation:
%\begin{equation}
%\Psi(i|\mathcal{P}) \approx \Psi(i|\emptyset)
%\end{equation}
%The approximation only consider the first-order effect as node's influence. $\Psi(i|\emptyset)$ could be calculated by adding each node in empty attack set $\emptyset$ in turn and get corresponding attack influence. We name it SPSA-GCG method.
\section{Experiments}
\label{sec:exp}
In this section, we evaluate the performance of the proposed diffusion attack algorithm under different scenarios using real-world data.

\subsection{Experiment Setup}
\label{sec:expset}
We examine the proposed attack algorithm on three traffic prediction models and two datasets, and details are described as follows:

{\bf Traffic Data.} We consider two real-world traffic datasets:
\begin{itemize}
    \item \texttt{LA}: The \texttt{LA} dataset contains traffic speed obtained from 207 loop detectors in Los Angeles \cite{zhao2019t}, and The data ranges from March 1st to March 7th, 2012. The average degree of the adjacency matrix is 14. The speed data are collected every 5 minutes, and the average speed is 58km/h. The study area is showed in the upper part of Fig.~\ref{fig:scatter}.
    \item \texttt{HK}: The \texttt{HK} speed data is collected from an open data platform initiated by the Hong Kong government, and overall 179 roads are considered in the Hong Kong island and Kowloon area. The data ranges from May 1st to May 31st, 2020. The average degree of the adjacency matrix is 2. The speed data are collected every 5 minutes, and the average speed is 45km/h. The study area is showed in the lower part of Fig.~\ref{fig:scatter}.
\end{itemize}

{\bf Traffic prediction models.} We evaluate the developed diffusion attack framework on three traffic prediction models: \textsc{T-Gcn} \cite{zhao2019t}, \textsc{St-Gcn} \cite{yu2018spatio}, and \textsc{A3t-Gcn} \cite{zhu2020a3t}, which are all based on GCN structures. For each model, we set $S=12$ and $T=1$. 
%Therefore the forecasting speed $y_i(\bm{X})$ is a scalar. 
To conduct the comprehensive evaluation, we train four variants of each model, which are the original model, the original model with \textsc{DropOut}, \textsc{DropNode}, and \textsc{DropEdge} regularization, respectively \cite{rong2020dropedge, dropconn}. For different drop regularization strategies, we set the drop probability to $30\%$. In Appendix~\ref{ap:train}, the Accuracy and Root Mean Squared Error (RMSE) of each model are showed in TABLE \ref{accuracy} and TABLE \ref{rmse}, and both measures are defined in \cite{zhao2019t}. Overall, accuracy of the different prediction models are around $90\%$ in testing data. 

{\bf Attack settings.} The parameter settings for the diffusion attack models and algorithms, as well as the evaluation criterion are discussed as follows:
\begin{itemize}
    \item {\em Model specifications.} In Equation~\ref{eq:diff}, we set the constraint $-z_i \bm{x}_{i} \leq \bm{u}_{i} \leq 0.5 z_i \bm{x}_{i} $, and $w_i = 1$, which means each node is equally important. The cost is defined as $b_i = \text{Degree}(i)= S_D(i)$, where $S_D(i) = \sum_{j} \bm{A}_{ij}$. $B = \{20, 50, 100, 150, 200\}$. The optimal target reduces to $\Phi(\bm{U}) = \sum_{i \in \mathcal{V}} \phi_i(\bm{U})$, where $\phi_i(\bm{U}) = y'_i(\bm{X}') - y_i (\bm{X})$. The attack algorithm aims to reduce the predicted speed and we intend to generate virtual ``congestion'' on the network. % which represents the degradation value of forecasting speed for node $i$. We aims to lower down the forecasting speed and create traffic congestion. 
    \item {\em Evaluation of Algorithms.} We use Average Attack Influence (AAI) and Average Attack Influence Ratio (AAIR) to evaluate the effect of the diffusion attacks. We define
\begin{equation}
    \begin{array}{lllllll}
      \text{AAI} &=& \frac{1}{N} \sum_{i \in \mathcal{V}} |\phi_i(\bm{U})|\\
       \text{AAIR} &=& \frac{1}{N} \sum_{i \in \mathcal{V}}\frac{\phi_i(\bm{U})}{|y_i(\bm{X})|}
    \end{array},\nonumber
\end{equation}
where AAI represents the average degradation and AAIR represents the average degradation ratio of the prediction on the entire network, respectively. %And we average the results on 5 time samples as final result. Each of time sample could be viewed as independent in 100 minutes interval. 
    \item {\em SPSA Setting.} For Algorithm~\ref{alg:spsa}, $a=0.328, c=0.1, \alpha=0.202, \gamma=0.101$, and $\eta=\frac{n}{10}$. For diffusion attack, $\texttt{MaxIter} = 30000$; for computing the $\hat{\phi}_i$, $\texttt{MaxIter} = 100$.
\end{itemize}

{\bf Baseline algorithms.} Because the diffusion attack is a newly proposed task, there are very few existing methods that can be used as baseline algorithms. In addition to the proposed \textsc{Kg-Spsa} approach, we modify and develop  8 algorithms for comparison. The major difference of each baseline algorithm lies in how to select the attack set $\mathcal{P}$ and whether to use the \textsc{Kg}-$\star$ greedy algorithm. \textsc{Degree} selects nodes with highest degree $S_D(i)$. \textsc{K-Medoids} selects nodes by clustering the nodes with geo-location features until reaching the total budget $B$ \cite{kaur2014k}, \textsc{Pagerank} selects nodes with highest pagerank scores  $S_{PR}(i) = \frac{1-\alpha}{N} + \alpha \sum_{j \in \mathcal{N}_i} \frac{S_{PR}(j)}{|\mathcal{N}_j|}$, where $\alpha = 0.85$, and \textsc{Betweenness} chooses nodes with high betweenness scores $S_{Bw}(i) = \sum_{j \neq i, k \neq i} \frac{\text{Path}_{jk}(i)}{\text{Path}_{jk}}$, where $\text{Path}_{jk}$ is the number of shortest path between node $j$ and $k$, and $\text{Path}_{jk}(i)$ is the number of shortest path that passes node $i$. The \textsc{Random} algorithm just selects the node randomly until meets the budget limit. \textsc{Spsa} selects the highest $\hat{\phi}_i = \phi_i(\bm{U}_{\mathcal{V}})$ by running Algorithm~\ref{alg:spsa} with $\mathcal{P} = \mathcal{V}$, and the greedy algorithm is not used in \textsc{Spsa}. When we use the greedy algorithm, it is also possible to use centrality measures such as pagerank and betweenness to represent $\hat{\phi}_i$. Different from \textsc{Pagerank} and \textsc{Betweenness}, which determine $\mathcal{P}$ by the highest scores, \textsc{Kg-Pagerank} and \textsc{Kg-Betweenness} approximate $\hat{\phi}_i$ by the pagerank and betweenness scores, followed by running Algorithm~\ref{alg:gcg} with different $\hat{\phi}_i$.

%in Pagerank-GCG algorithm, we select nodes with high score $S_{PRG}(i) \triangleq \frac{S_{PR}(i)}{S_{D}(i)}$, which is divided by its degree.Respectively, Betweenness-GCG algorithm selects nodes with high score $S_{BWG}(i) \triangleq \frac{S_{Bw}(i)}{S_D(i)}$, which is also divided by its degree. In Random algorithm, we select nodes randomly. As comparison, in algorithm \ref{alg:gcg}, we set $max\_utility = \Psi(i|\mathcal{P})$, and name it SPSA method.

\subsection{Results of diffusion attack}
In this section, we present the experimental results. We first verify the diffusion effect when attacking a single node, then the performance of the proposed attack algorithm is evaluated and compared with baseline methods in different scenarios. Lastly, we examine the robustness of the proposed algorithm under different drop regularization strategies. 
\subsubsection{Diffusion Effects of attacks on a single node}
%As Fig~\ref{fig:property} shows, 
To demonstrate the diffusion phenomenon, we construct an attack on a single node for the three traffic prediction models, and the attack effect of different hops of neighbors is presented in Fig.~\ref{fig:property}. As can be seen, the attack mainly influences the ego node and its 1$\sim$2-hop neighbors, and the influence will diminish as the number of hops increases. As proven in Proposition~\ref{prop:diff}, for a $L$-layer GCN model, the diffusion only occurs within $L$-hop neighbors, which is verified in this experiment. The experimental results also indicate that the influence of a single node attack is localized, and a successful diffusion attack algorithm requires a scattered node selection strategy.

%indicates . With a scattered strategy, the overlap of attack influence could almost be ignored. In other word, the diffusion attack present a linear behavior, or first-order effect. It may count for the success of our approximation (SPSA-GCG) . 

\begin{figure}[ht]
\centering
\includegraphics[height=4.2cm,width=7cm]{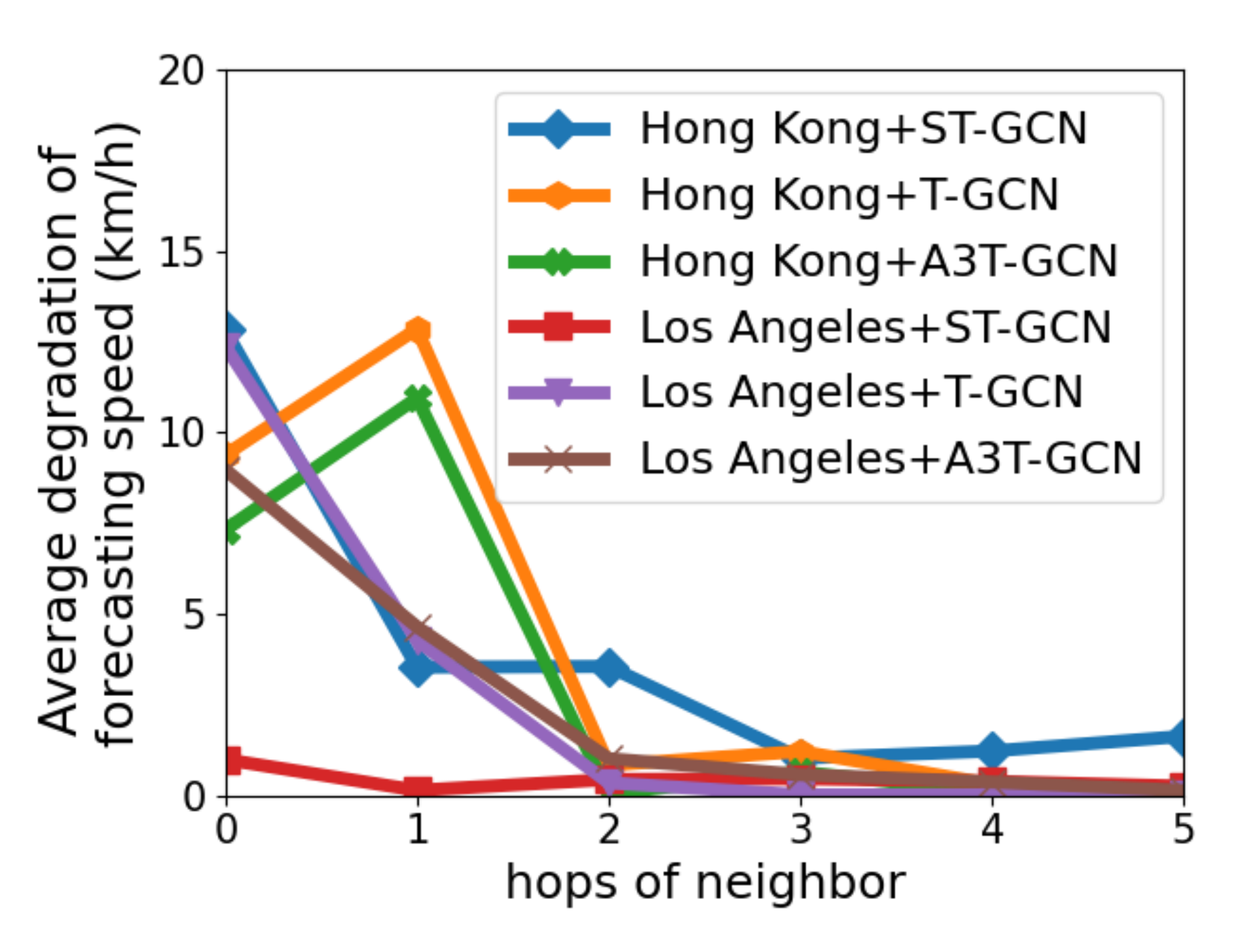}
\caption{Diffusion effects when attacking on a single node.}
\label{fig:property}
\end{figure}

\subsubsection{Comparisons of different attack algorithms} 

The proposed algorithm \textsc{Kg-Spsa} is compared with different baseline algorithms with $B = 50$, and the corresponding AAI is presented in TABLE~\ref{tbl:comparison} (unit: km/hour), and the AAIR table is presented in Appendix~\ref{sec:more}. The attack algorithms are categorized into two types: semi-black-box algorithms that know the graph structure, and black-box algorithms that only require inputs and outputs of the prediction models. From TABLE~\ref{tbl:comparison} one can see, the proposed algorithm \textsc{Kg-Spsa} outperforms all the baseline algorithms on \texttt{LA}, and \textsc{Kg}-$\star$ generally outperforms the original counterparts. Comparing with other methods, \textsc{Kg-Spsa} is a black-box method, which does not rely on knowledge of the graph structure ({\em i.e.} adjacency matrix $\bm{A}$). In real-world, it is challenging to obtain the information of $\bm{A}$ in the prediction system as the graph can be generated by different configurations such as sensor layout, network topology, and causal relationship, and this information is hidden to users \cite{Ye_2020}. Fig.~\ref{fig:scatter} presents the selected nodes by \textsc{Kg-Spsa} for different prediction methods with $B=50$, and the color (from green to red) represents AAIR of each node under the attack. To maximize the attack effect, the selected nodes distribute across the entire network, which is consistent with our previous conjecture.

\begin{table*}[ht]
\centering
\caption{Comparison of different diffusion attack algorithms in terms of AAI. ($B=50$)}
\begin{tabular}{llllllll} % columns
\toprule[1pt]
\multirow{2}{*}{Types} & \multirow{2}{*}{Algorithm} &  \multicolumn{3}{c}{
\texttt{LA}} &  \multicolumn{3}{c}{\texttt{HK}} \\
\cline{3-8}  %line between 4-8 column
 & & \textsc{St-Gcn} & \textsc{T-Gcn} & \textsc{A3t-Gcn}&  \textsc{St-Gcn} & \textsc{T-Gcn} & \textsc{A3t-Gcn}\\
\midrule
\multirow{7}{*}{Semi-black-box} & \multirow{1}{*}{\textsc{Degree}}      &0.74   &0.65   &0.64   &5.30   &3.64   &4.22 \\
\cline{2-8}
&\multirow{1}{*}{\textsc{K-Medoids}}       &0.69  &0.87   &1.69   &12.41    &5.72   &7.74  \\
\cline{2-8}
&\multirow{1}{*}{\textsc{Pagerank}}        &2.41  &2.11   &2.61   &9.42   &4.84   &5.54 \\
\cline{2-8}
&\multirow{1}{*}{\textsc{Betweenness}}     &2.35  &1.56   &2.06   &16.28   &7.84   &19.88 \\
\cline{2-8}
&\multirow{1}{*}{\textsc{Kg-Betweenness}} &2.75  &1.92   &2.66   &17.37   &8.42   &24.44 \\
\cline{2-8}
&\multirow{1}{*}{\textsc{Kg-Pagerank}}    &4.74  &3.69   &5.06   &22.63   &$\bm{12.99}$  &$\bm{34.47}$ \\
\cline{1-8}
% \midrule[1pt]
\multirow{4}{*}{Black-box}& \multirow{1}{*}{\textsc{Random}}   &1.06  &1.31   &1.86   &14.61   &8.74   &20.42 \\
\cline{2-8}
& \multirow{1}{*}{\textsc{Spsa}}                              &3.18  &1.46   &7.66   &18.60   &7.43   &28.97 \\
\cline{2-8}
% &\multirow{1}{*}{SPSA-GCG-S}                         &4.86  &3.69   &6.45   &22.47   &12.26   &29.65 \\
% \cline{2-8}
&\multirow{1}{*}{\textsc{Kg-Spsa}}                           & $\bm{5.46}$  &$\bm{4.36}$   &$\bm{12.74}$   &$\bm{23.34}$ & {\em12.21} & {\em32.86} \\
\bottomrule[1pt]
  \end{tabular}
\label{tbl:comparison}
\end{table*}

\begin{figure*}[ht]
\centering
\includegraphics[height=7cm,width=18.15cm]{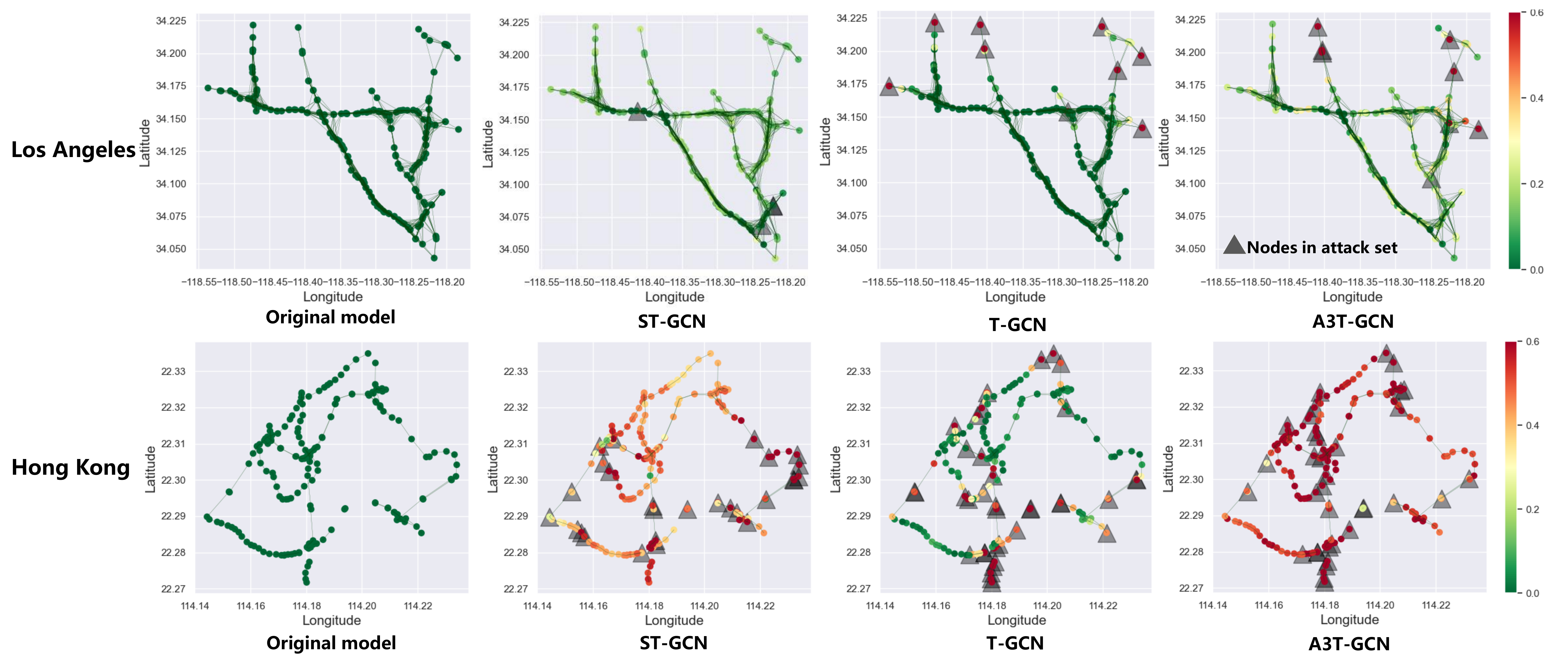}
\caption{The distribution of selected nodes and AAIR of each node under the attack algorithm \textsc{Kg-Spsa}. (the selected nodes are marked as triangle, and the color represents AAIR)}
\label{fig:scatter}
\end{figure*}

In addition, it is observed that the robustness of the three prediction models is different. \textsc{A3t-Gcn} demonstrates great vulnerability under attack algorithms, while both \textsc{T-Gcn} and \textsc{St-Gcn} are more robust. This could be due to the strengthened connection among nodes by the attention layers in \textsc{A3t-Gcn}, meanwhile, the strengthened connection can also increase the vulnerability of the prediction models. Comparing the dataset \texttt{LA}, predictions on \texttt{HK} are more vulnerable to adversarial diffusion attacks, which might be explained by the drastic changes of Hong Kong's traffic conditions within a day \cite{tam2011application}.

\subsubsection{Sensitivity analysis on the budget $B$}
To study the effect of budget $B$ on the diffusion attack, we run \textsc{Kg-Spsa} with $B \in [20, 50, 100, 150, 200]$ for both \texttt{LA} and \texttt{HK}, and the corresponding AAI is presented in Fig.~\ref{fig:budget}. The attack influence will increase when the total budget $B$ increases for both datasets while the trend is becoming marginal. It is also observed that prediction models on \texttt{HK} are more vulnerable, while \textsc{A3t-Gcn} is the least robust predictions models for both datasets.

\begin{figure}[ht]
\centering
\includegraphics[height=4.3cm,width=8.85cm]{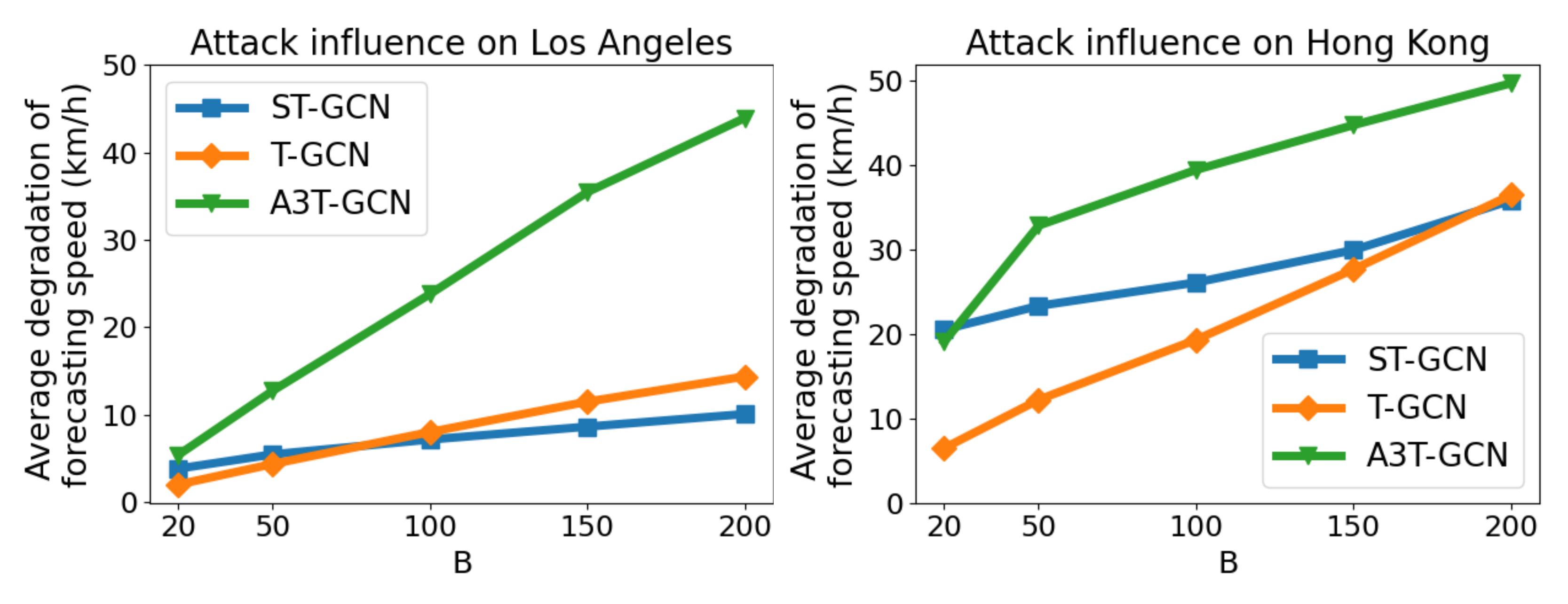}
\caption{Attack effect with different budget $B$.}
\label{fig:budget}
\end{figure}

\subsubsection{Performance of the proposed algorithm on drop regularization strategies}
To better understand the performance of the proposed attack algorithms, we examine the attack effects on the prediction models with drop regularization. Existing studies have widely demonstrated that drop regularization strategies could improve the robustness of the GCN-based model \cite{rong2020dropedge, dropconn}, hence it is crucial to show that the performance of the proposed methods remains effective under different drop regularization strategies. To this end, we conduct diffusion attacks with \textsc{Kg-Spsa} on the prediction models trained with \textsc{DropOut}, \textsc{DropNode}, and \textsc{DropEgde}. \textsc{DropOut} randomly drops rows in feature matrix $\bm{X}$, \textsc{DropNode} randomly drops a subset of the nodes on the graph, and \textsc{DropEgde} will randomly drop the edges of the graph for each epoch during the model training. Details of the models are presented in section~\ref{sec:expset}. We run \textsc{Kg-Spsa} and \textsc{Kg-Pagerank} for prediction models with different drop regularization on the two datasets, and the algorithm performance is presented in TABLE~\ref{tbl:drop}. The reason we choose \textsc{Kg-Spsa} and \textsc{Kg-Pagerank} is because both algorithms outperform other semi-black-box and black-box algorithms. For dataset \texttt{LA}, \textsc{Kg-Spsa} outperforms \textsc{Kg-Pagerank} on all the prediction models, and \textsc{Kg-Spsa} is slightly better on \texttt{HK}. In most cases, \textsc{DropOut} could degrade the performance of the attack algorithms, while the other two drop regularization strategies do not protect the prediction models. Overall, the proposed diffusion attack algorithms could still generate adversarial samples under various drop regularization strategies.

%We set the drop probability of the tree apporaches to be $0.3$

%For the latter two drop method on graph structure, we use random sampler strategy. The result are showed in TABLE 3. Our method is still effective when deal with these defense models.

\begin{table}[ht]
\centering
\caption{Comparison of \textsc{Kg-Spsa} and \textsc{Kg-Pagerank} on the three defense strategies in terms of AAI. ($B=50$)}
\label{tbl:drop}
\resizebox{0.5\textwidth}{!}{
\begin{tabular}{lll|lll} % column of the table
\toprule
Datasets  & Model & Baseline & \textsc{DropOut} & \textsc{DropNode}  & \textsc{DropEdge} \\

\midrule
\multicolumn{6}{c}{\textsc{Kg-Spsa}}\\
\cline{1-6}
\multirow{3}{*}{\texttt{LA}} &\textsc{St-Gcn} & {\bf5.46}  & {\bf 5.65} & {\bf 6.86} & {\bf 11.60} \\
\cline{2-6}
                     &\textsc{T-Gcn}  & {\bf 4.36}  & {\bf 3.43} & {\bf 3.49} & {\bf 2.79} \\
\cline{2-6}
                     &\textsc{A3t-Gcn} &{\bf 12.74} & {\bf 3.07} & {\bf 18.59} & {\bf 6.18} \\
\cline{1-6}
\multirow{3}{*}{\texttt{HK}}  &\textsc{St-Gcn} & {\bf 23.34}  &11.89 &14.53 & {\bf 25.84} \\
\cline{2-6}
                     &\textsc{T-Gcn}&12.21 & {\bf 12.34} &11.63 &11.43 \\
\cline{2-6}
                     &\textsc{A3t-Gcn} &32.86 & {\bf 41.77} &11.44 & {\bf 91.77} \\
\cline{1-6}                    
\multicolumn{6}{c}{\textsc{Kg-Pagerank}}\\
\cline{1-6}
\multirow{3}{*}{\texttt{LA}} &\textsc{St-Gcn}&4.74 &2.71 &4.65 &6.67 \\
\cline{2-6}
                     &\textsc{T-Gcn}&3.69 &3.06 &3.21 &2.47 \\
\cline{2-6}
                     &\textsc{A3t-Gcn} &5.06 &2.68 &6.07 &3.42 \\
\cline{1-6}
\multirow{3}{*}{\texttt{HK}}  &\textsc{St-Gcn} &22.63 &{\bf 12.84} &{\bf 16.02} &24.84 \\
\cline{2-6}
                     &\textsc{T-Gcn} & {\bf 12.99} &{\bf 12.34} &{\bf 12.31} & {\bf 12.04} \\
\cline{2-6}
                     &\textsc{A3t-Gcn} &{\bf 34.47} &28.34 & {\bf 24.10} &74.93 \\
% \cline{1-6}                    
% \multicolumn{6}{c}{Betweenness-GCG}\\
% \cline{1-6}
% \multirow{3}{*}{Los} &ST-GCN &2.75 &1.43 &2.49 &8.09 \\
% \cline{2-6}
%                      &T-GCN &1.92 &1.44 &1.52 &1.22 \\
% \cline{2-6}
%                      &A3T-GCN &2.66 &1.33 &2.83 &1.83 \\
% \cline{1-6}
% \multirow{3}{*}{HK}  &ST-GCN &17.37 &10.64 &11.40 &16.16 \\
% \cline{2-6}
%                      &T-GCN &8.42 &7.47 &7.91 &7.62 \\
% \cline{2-6}
%                      &A3T-GCN &24.44 &14.05 &4.70 &20.78 \\
\bottomrule
  \end{tabular}
  }
\end{table}

% \begin{table*}[!t] \caption{Attack Evaluation Of K-hop Neighbor}
% \centering
% \begin{tabular}{lllllll}
% \toprule 
%          &            &Los       &        &       &Sz        &     \\  
%   K-hop  &Node Number &AE/$\%$  &Average AE/$\%$ &Node Number &AE/$\%$  &Average AE/$\%$\\  
% \midrule
%   Self   &1   &6.4    &6.4     &1     &3.3    &3.3      \\
%   K=1    &18  &33.7   &1.9     &3     &19.0   &6.3      \\
%   K=2    &21  &8.9    &0.4     &6     &14.0   &2.3      \\
%   K=3    &19  &1.9    &0.1     &6     &5.7    &0.95     \\
%   K=4    &24  &1.8    &0.08    &6     &2.7    &0.45     \\
%   K=5    &30  &2.5    &0.08    &6     &7.8    &1.3      \\
%   Others &94  &12.5   &0.13    &128   &174.4  &1.36     \\ 
% \bottomrule 
% \label{table_time} 
% \end{tabular}
% \end{table*}

\subsection{Discussions}
In this section, we discuss the implications and suggestions for improving the robustness of the traffic prediction models. In the previous section, we carry out numerical experiments to demonstrate the performance of the proposed attack algorithms on different datasets, prediction models, and regularization strategies. Based on the experimental results, we provide the following suggestions to improve the model robustness during different phrases:
\begin{itemize}
    \item {\bf Model selection.} When choosing GCN-based models for speed prediction, RNN-based models are generally more robust than attention-based models. Depending on the data and city scale, it is suggested to choose models with simpler layers, as the complex layers in \textsc{St-Gcn} and \textsc{A3t-Gcn} can sometimes degrade significantly under attacks. There is a trade-off between accuracy and robustness, so it is critical to balance the accuracy and robustness for practical usage. 
    \item {\bf Model regularization.} Based on the experimental results, it is suggested to adopt \textsc{DropOut} during the training, as the model accuracy remains high while the robustness can be improved after the \textsc{DropOut} training. It is also suggested to test different drop regularization strategies before the actual deployment.
    \item {\bf Model privacy.} The graph structure should not be disclosed to the public, as it can significantly improve the efficiency of the attack algorithms. It is also suggested to frequently update the prediction model, as the attack models rely on multiple trials and errors on the prediction models. If the prediction model updates frequently, then the robustness of the entire prediction system can be significantly improved.  
    \item {\bf Active defending strategies.} Before actual deployment, it is necessary to comprehensively test the vulnerability of the prediction models and to identify the critical nodes with significant attack influence. For those important nodes, we can enhance the protection by regular patrol in the physical world and consistency checking in the cyber system. For example, if an RSU on a road segment is identified to be critical, then this device should be protected physically \cite{9273042}. If the attack on this device indeed occurs, the traffic center should spot the anomaly in real-time and block the information sent from this device. %In addition, mobile users who are connecting the platform with this RSU will be under suspicion. 
\end{itemize}

% \subsubsection{Model selection}

% \subsubsection{Defending strategies during training}

% \subsubsection{Defending strategies after training}

\section{Conclusion}
\label{sec:con}
In this paper, we explore the robustness and vulnerability issues of graph-based neural network models for traffic prediction. Different from existing adversarial attack tasks, adversarial attacks for traffic prediction require to degrade the model performance for the entire network, rather than a specific sample of nodes. Given this, we propose a novel concept of diffusion attack, which aims to reduce the prediction accuracy of the whole traffic network by perturbing a small number of nodes. To solve for the diffusion attack task, we develop an algorithm \textsc{Kg-Spsa}, which consists of two major components: 1) using SPSA to generate the optimal perturbations to maximize the attack effects; 2) adapting the greedy algorithm in the knapsack problem to select the most critical nodes. The proposed algorithm is examined with three widely used GCN-based traffic prediction models (\textsc{St-Gcn}, \textsc{T-Gcn}, and \textsc{A3t-Gcn}) on the Los Angeles and Hong Kong datasets. The experimental results indicate that the proposed algorithm outperforms the baseline algorithms under various scenarios, which demonstrates the effectiveness and efficiency of the proposed algorithm. In addition, the proposed attack algorithms can still generate effective adversarial samples for traffic prediction models trained with drop regularization. This study could help the public agencies and private sectors better understand the robustness and vulnerability of GCN-based traffic prediction models under adversarial attacks, and strategies to improve the model robustness in different phrases are also discussed.

As for the future research directions, the proposed attack algorithms could be applied to not only road traffic prediction, but also other traffic modes such as urban railway transit systems, ride-sourcing services, and parking systems \cite{yang2019deep, zhang2021short}. It is also important to study the effect of adversarial attacks on flow prediction, origin-destination demand prediction, and other tasks relying on the GCN-based models. For the users of the traffic prediction models, it is critical to develop models for defending adversarial attacks and protecting traffic prediction results. Another way to protect the prediction model is through real-time anomaly detection and filtering of the incoming data stream, which could be a new research direction for improving the robustness of the traffic prediction models under adversarial attacks.

\section*{Supplementary Materials}
The proposed diffusion attack algorithm and evaluation framework are implemented in Python and open-sourced on GitHub (\url{https://github.com/LYZ98/Adversarial-Diffusion-Attacks-on-Graph-based-Traffic-Prediction-Models}).

\section*{ACKNOWLEDGMENT}
The work described in this study was supported by a grant funded by the Hong Kong Polytechnic University (Project No. P0033933). The contents of this paper reflect the views of the authors, who are responsible for the facts and the accuracy of the information presented herein. 

%These models use temporary and spatial information for forecasting, and it is quiet different from traditional GCN models which are used for node classification task. We come up with a new attack concept- Diffusion Attack, which means only using limited nodes to attack the whole traffic system. To achieve it, we construct black-box attack by SPSA and greedy algorithm. The algorithm is effective on defense models and has a greater effect than other algorithms. This result shows the security of Graph-based traffic models is still worth thinking.

\appendices

\section{More Details about the three traffic prediction models and attack results}
\label{ap:train}

To train the three traffic prediction models, we set the learning rate to be $0.001$, batch size to be $32$, and the number of epoch to be  $300$. The two datasets are divided into two parts, in which $80\%$ and $20\%$ are training set and testing set, respectively. Mean Squared Error (MSE) is used as the loss function \cite{zhao2019t}, and Adam is adopted as the optimizer. The testing accuracy in terms of accuracy and Root Mean Squared Error (RMSE) of the trained prediction models are presented in TABLE~\ref{accuracy} and TABLE~\ref{rmse}, respectively. Overall, all the prediction models could achieve high prediction accuracy on both datasets.

\begin{table}[ht]
\centering
\caption{Accuracy of trained models}
\label{accuracy}
\resizebox{0.5\textwidth}{!}{
\begin{tabular}{llllll} % column of the table
\toprule
Dataset  & Model  & Baseline & \textsc{Dropout} & \textsc{Dropnode}  & \textsc{Dropedge} \\
\midrule
\multirow{3}{*}{\texttt{LA}} &\textsc{St-Gcn} &92.68\% &92.70\% &92.77\% &88.18\% \\
\cline{2-6}
                     &\textsc{T-Gcn}  &90.44\% &89.24\% &90.01\% &90.05\% \\
\cline{2-6}
                     &\textsc{A3t-Gcn} &89.04\% &72.71\% &89.15\% &89.84\% \\
\cline{1-6}
\multirow{3}{*}{\texttt{HK}}  &\textsc{St-Gcn} &92.88\% &93.20\% &93.17\% &92.80\% \\
\cline{2-6}
                     &\textsc{T-Gcn} &88.50\% &87.70\% &86.78\% &88.08\% \\
\cline{2-6}
                     &\textsc{A3t-Gcn} &89.47\% &87.30\% &80.57\% &88.12\% \\
\bottomrule
\end{tabular}
}
\end{table}

\begin{table}[ht]
\centering
\caption{RMSE of trained models}
\label{rmse}
\resizebox{0.5\textwidth}{!}{
\begin{tabular}{llllll} % column of the table
\toprule
Dataset  & Model  & Baseline & \textsc{DropOut} & 
\textsc{DropNode}  & \textsc{DropEdge} \\

\midrule
\multirow{3}{*}{\texttt{LA}} &\textsc{St-Gcn} &4.30 &4.28 &4.25 &6.95 \\
\cline{2-6}
                     &\textsc{T-Gcn}  &5.61 &6.32 &5.86 &5.84 \\
\cline{2-6}
                     &\textsc{A3t-Gcn} &6.43 &16.02  &6.36 &5.96 \\
\cline{1-6}
\multirow{3}{*}{\texttt{HK}}  &\textsc{St-Gcn} &3.55 &3.39 &3.41 &3.59 \\
\cline{2-6}
                     &\textsc{T-Gcn} &5.74 &6.13 &6.59 &5.95 \\
\cline{2-6}
                     &\textsc{A3t-Gcn} &5.25 &6.33 &9.68 &5.92 \\
\bottomrule
  \end{tabular}
  }
\end{table}

\section{Comparisons of different attack algorithms in terms of AAIR. ($B=50$)}
\label{sec:more}

Similar to TABLE~\ref{tbl:comparison}, we evaluate the performance of different attack algorithms in terms of AAIR in TABLE~\ref{tbl:comparison2}. In general, similar arguments could be obtained based on the AAIR, and the proposed \textsc{Kg-Spsa} outperforms other baseline models in \texttt{LA}, and performs similarly as the semi-black-box algorithms in \texttt{HK}.

\begin{table*}[t]
\centering
\caption{Comparison of different diffusion attack algorithms in terms of AAIR. ($B=50$)}
\begin{tabular}{llllllll} % columns
\toprule[1pt]
\multirow{2}{*}{Types} & \multirow{2}{*}{Algorithm} &  \multicolumn{3}{c}{
\texttt{LA}} &  \multicolumn{3}{c}{\texttt{HK}} \\
\cline{3-8}  %line between 4-8 column
 & & \textsc{St-Gcn} & \textsc{T-Gcn} & \textsc{A3t-Gcn}& \textsc{St-Gcn} & \textsc{T-Gcn}  &\textsc{A3t-Gcn}\\
\midrule
\multirow{7}{*}{Semi-blackbox} & \multirow{1}{*}{\textsc{Degree}}      &1.54\%   &1.88\%   &2.07\%   &12.35\%   &8.69\%   &10.28\% \\
\cline{2-8}
&\multirow{1}{*}{\textsc{K-Medoids}}   &1.37\%  &1.77\%   &2.33\%   &25.19\%    &11.84\%   &17.09\%  \\
\cline{2-8}
&\multirow{1}{*}{\textsc{Pagerank}}   & 3.79\% &3.77\%   &3.80\%   &19.90\%   &10.77\%   &12.57\% \\
\cline{2-8}
&\multirow{1}{*}{\textsc{Betweenness}}   & 3.83\% &3.70\%   &3.34\%   &30.95\%   &14.46\%   &43.92\% \\
\cline{2-8}
&\multirow{1}{*}{\textsc{Kg-Betweenness}}    & 4.73\% &4.21\%   &2.55\%   &32.69\%   &15.07\%   &52.50\% \\
\cline{2-8}
&\multirow{1}{*}{\textsc{Kg-Pagerank}}   &7.88\% &6.99\%   &8.38\%   &42.27\%   &23.91\%   &$\bm{72.37}$\% \\
\cline{1-8}
% \midrule[1pt]
\multirow{4}{*}{Blackbox}& \multirow{1}{*}{\textsc{Random}}    & 1.65\% &2.51\%   &3.19\%   &28.57\%   &16.63\%   &45.33\% \\
\cline{2-8}
& \multirow{1}{*}{\textsc{Spsa}}   & 5.80\% &3.07\%   &12.36\%   &35.61\%   &15.00\%   &60.63\% \\
% \cline{2-8}
% &\multirow{1}{*}{SPSA-GCG-S}  & 7.85\% &6.73\%   &9.59\%   &42.39\%   &22.97\%   &62.36\% \\
\cline{2-8}
&\multirow{1}{*}{\textsc{Kg-Spsa}}  & $\bm{8.32}$\% &$\bm{7.76}$\%   &$\bm{22.77}$\%   &$\bm{43.25}$\%   &$\bm{24.26}$\%   &70.21\% \\
\bottomrule[1pt]
  \end{tabular}
\label{tbl:comparison2}
\end{table*}

Results in TABLE~\ref{tbl:drop2} follow the same patterns as in TABLE~\ref{tbl:drop}, and one can observe that the proposed attack algorithms can still generate effective adversarial samples in terms of AAIR.

\begin{table}[htbp]
\centering
\caption{Comparison of \textsc{Kg-Spsa} and \textsc{Kg-Pagerank} on the three defense strategies im terms of AAIR. ($B=50$)}
\label{tbl:drop2}
\resizebox{0.5\textwidth}{!}{
\begin{tabular}{lll|lll} % column of the table
\toprule
Dataset  & Model &Baseline & \textsc{DropOut} & \textsc{DropNode}  & \textsc{DropEdge} \\

\midrule
\multicolumn{6}{c}{\textsc{Kg-Spsa}}\\
\cline{1-6}
\multirow{3}{*}{\texttt{LA}} &\textsc{St-Gcn} &{\bf 8.32}\% &{\bf 8.73}\% &{\bf 11.36}\% &{\bf 13.88}\% \\
\cline{2-6}
                     &\textsc{T-Gcn}  &{\bf7.76}\% &{\bf 6.21}\% &{\bf 6.15}\% &{\bf 4.98}\% \\
\cline{2-6}
                     &\textsc{A3t-Gcn} &{\bf22.77}\% &{\bf 8.33}\% &{\bf38.11}\% &{\bf14.35}\% \\
\cline{1-6}
\multirow{3}{*}{\texttt{HK}}  &\textsc{St-Gcn} &{\bf43.25}\% &23.66\% &27.51\% &{\bf46.36}\% \\
\cline{2-6}
                     &\textsc{T-Gcn} &{\bf24.26}\% &{\bf23.47}\% &{\bf22.83}\% &22.75\% \\
\cline{2-6}
                     &\textsc{A3t-Gcn} &70.21\% &{\bf79.80}\% &24.34\% &{\bf230.85}\% \\
\cline{1-6}                    
\multicolumn{6}{c}{\textsc{Kg-Pagerank}}\\
\cline{1-6}
\multirow{3}{*}{\texttt{LA}} &\textsc{St-Gcn} &7.88\% &4.47\% &7.84\% &9.38\% \\
\cline{2-6}
                     &\textsc{T-Gcn} &6.99\% &5.60\% &5.71\% &4.44\% \\
\cline{2-6}
                     &\textsc{A3t-Gcn} &8.38\% &6.93\% &12.08\% &6.60\% \\
\cline{1-6}
\multirow{3}{*}{\texttt{HK}}  &\textsc{St-Gcn} &42.27\% &{\bf24.47}\% &{\bf29.25}\% &45.24\% \\
\cline{2-6}
                     &\textsc{T-Gcn} &23.91\% &23.09\% &22.07\% &{\bf22.86}\% \\
\cline{2-6}
                     &\textsc{A3t-Gcn} &{\bf72.37}\% &52.71\% &{\bf56.86}\% &174.56\% \\
\bottomrule
  \end{tabular}
  }
\end{table}

\ifCLASSOPTIONcaptionsoff
  \newpage
\fi

% \clearpage  % add by self
% % references section

\bibliographystyle{IEEEtran}
\bibliography{ref}

% \end{thebibliography}


\begin{IEEEbiography}[{\includegraphics[width=1in,height=1.25in,clip,keepaspectratio]{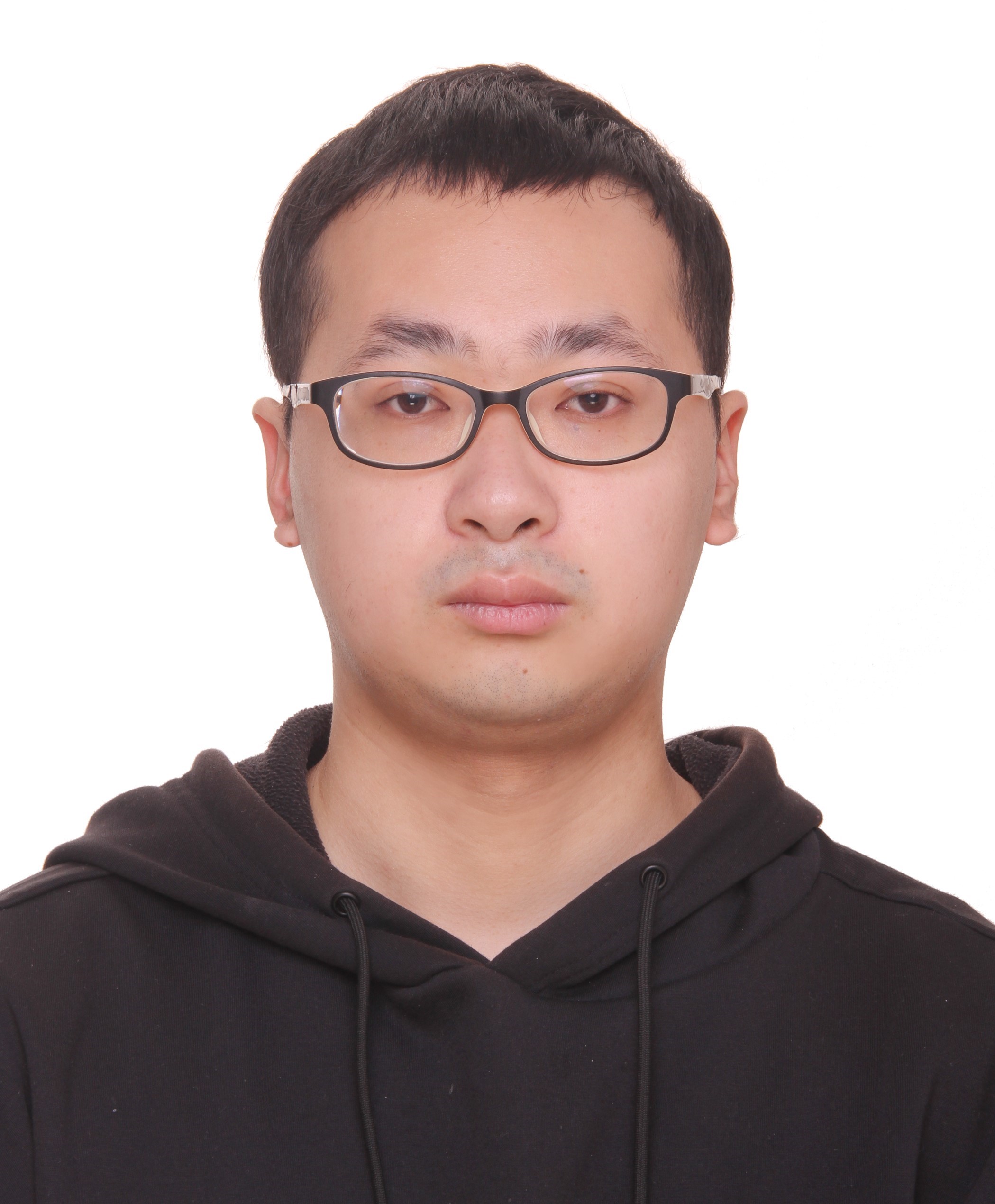}}]{Lyuyi ZHU} is an undergraduate student from College of Civil Engineering and Architecture, Zhejiang University, Hangzhou, China. He will join the School of Data Science, City University of Hong Kong as a PhD student. His research interest includes machine learning, optimization and numerical method. 
\end{IEEEbiography}

\begin{IEEEbiography}[{\includegraphics[width=1in,height=1.25in,clip,keepaspectratio]{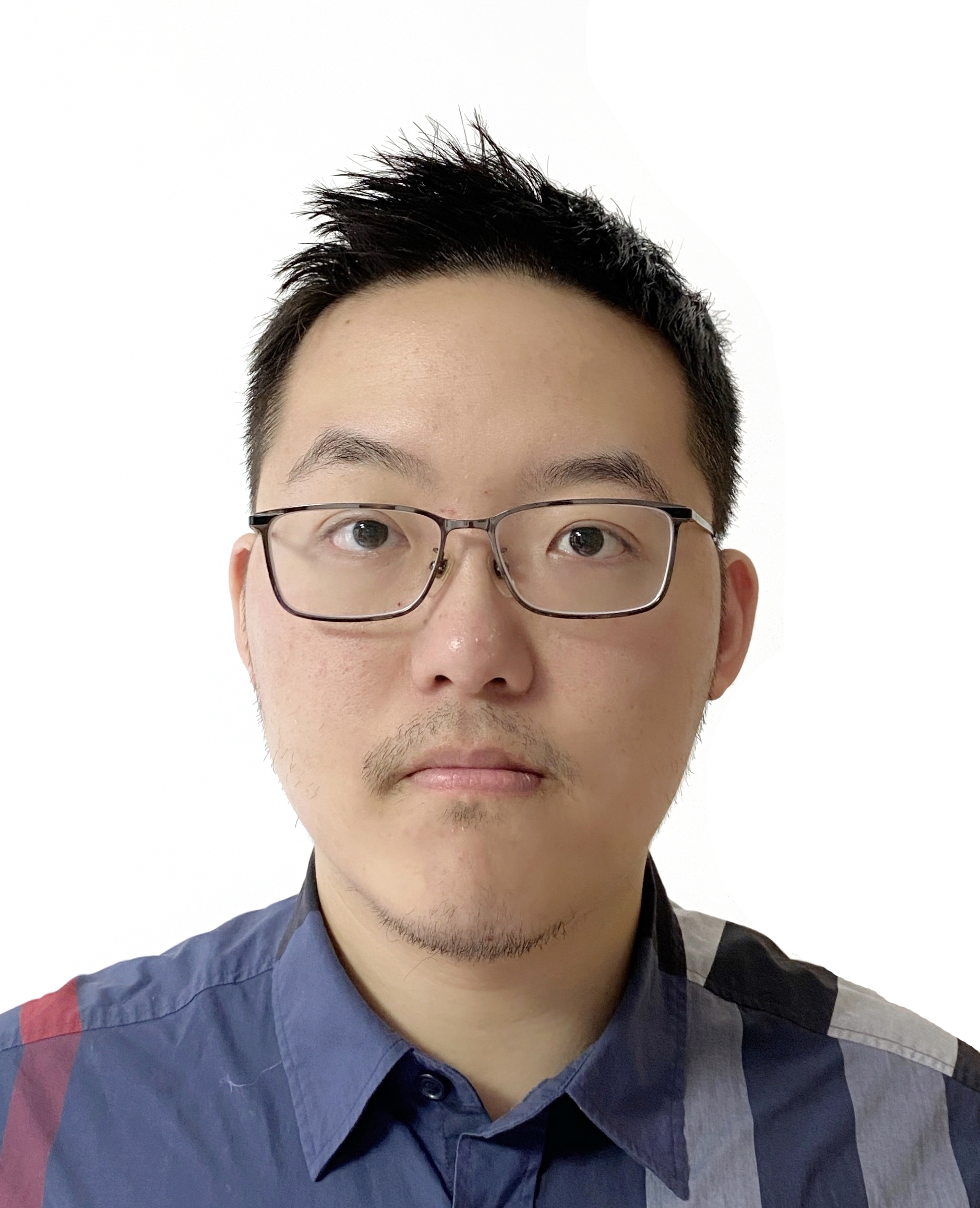}}]{Kairui Feng} is currently a Ph.D. student from Department of Civil and Environmental Engineering, Princeton University, New Jersey, USA. He received bachelor’s degrees in Civil Engineering and Mathematics from Tsinghua University, China.  His research interest includes infrastructure system modeling/optimization and climate change using data-driven and numerical approaches. 

\end{IEEEbiography}

\begin{IEEEbiography}[{\includegraphics[width=1in,height=1.25in,clip,keepaspectratio]{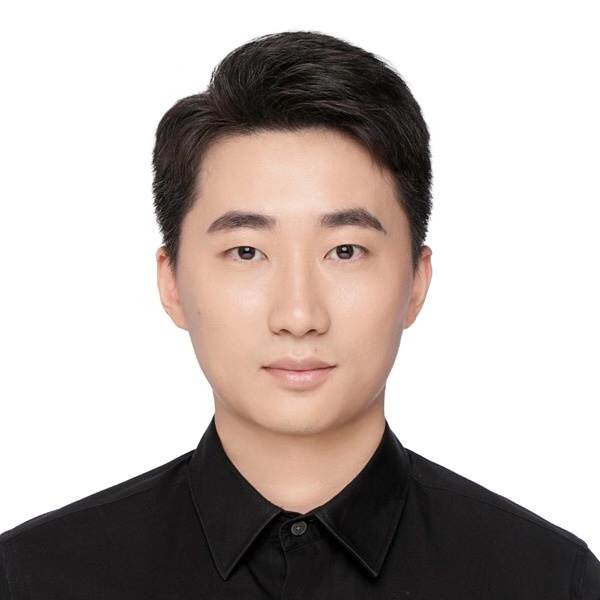}}]{Ziyuan Pu} is currently a Lecturer (Assistant Professor) at Monash University. He received B.S. degree in transportation engineering in 2010 at Southeast University, China. He received M.S. and Ph.D. degree in civil and environmental engineering in 2015 and 2020, respectively, at the University of Washington, US. His research interest includes transportation data science, smart transportation infrastructures, connected and autonomous vehicles (CAV), and urban computing.
\end{IEEEbiography}

%\newpage

\begin{IEEEbiography}[{\includegraphics[width=1in,height=1.25in,clip,keepaspectratio]{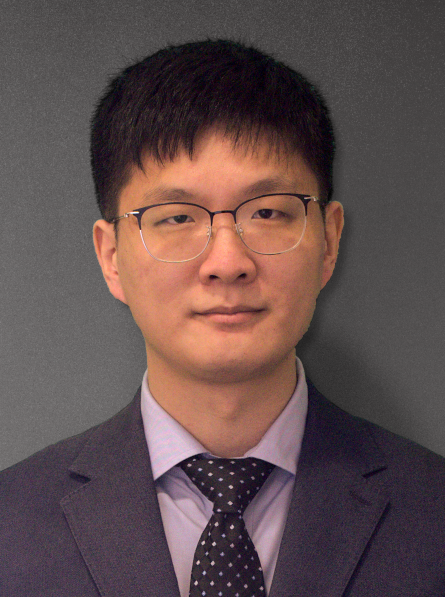}}]{Wei Ma} received bachelor’s degrees in Civil Engineering and Mathematics from Tsinghua University, China, master degrees in Machine Learning and Civil and Environmental Engineering, and PhD degree in Civil and Environmental Engineering from Carnegie Mellon University, USA. He is currently an assistant professor with the Department of Civil and Environmental Engineering at the Hong Kong Polytechnic University (PolyU). His research focuses on intersection of machine learning, data mining, and transportation network modeling, with applications for smart and sustainable mobility systems. He has received awards for research excellence and his contributions to the area, including 2020 Mao Yisheng Outstanding Dissertation Award, and best paper award (theoretical track) at INFORMS Data Mining and Decision Analytics Workshop.
\end{IEEEbiography}


%\vfill


\end{document}